\documentclass[12pt]{colt2014} 
\usepackage{framed,epsf,latexsym,amscd,textcomp,algorithmic}

\newtheorem{counter-example}[theorem]{Counter example}
\newtheorem{open question}[theorem]{Open question}

\newtheorem{claim}{Claim}

\newcommand{\ignore}[1]{}

\newcommand{\ca}{{\cal A}}

\newcommand{\cd}{{\cal D}}

\newcommand{\cg}{{\cal G}}
\newcommand{\ch}{{\cal H}}

\newcommand{\ci}{{\cal I}}

\newcommand{\cf}{{\cal F}}

\newcommand{\cu}{{\cal U}}
\newcommand{\cx}{{\cal X}}
\newcommand{\cy}{{\cal Y}}

\newcommand{\md}{\mathrm {md}}

\DeclareMathOperator*{\sign}{sign}
\DeclareMathOperator*{\argmax}{argmax}
\DeclareMathOperator*{\Ndim}{Ndim}
\DeclareMathOperator*{\Gdim}{Gdim}

\DeclareMathOperator*{\VCdim}{VCdim}
\DeclareMathOperator*{\Com}{Com}
\DeclareMathOperator*{\DeCom}{DeCom}

\DeclareMathOperator*{\conv}{conv}
\newcommand{\pac}{\mathrm{PAC}}
\newcommand{\erm}{\mathrm{ERM}}

\newcommand{\reals}{{\mathbb R}}

\DeclareMathOperator{\Err}{Err}

\DeclareMathOperator*{\E}{\mathbb{E}}

\newcommand{\inner}[1]{\langle #1 \rangle}

\newcommand{\secref}[1]{Section~\ref{#1}}
\newcommand{\thmref}[1]{Theorem~\ref{#1}}
\newcommand{\lemref}[1]{Lemma~\ref{#1}}

\title[Multiclass Learning]{Optimal Learners for Multiclass Problems}

\coltauthor{\Name{Amit Daniely} \addr Dept. of Mathematics, The Hebrew
  University, Jerusalem, Israel 
\AND 
\Name{Shai Shalev-Shwartz} \addr School of Computer Science and Engineering, The Hebrew University, Jerusalem, Israel
}

\begin{document}

\maketitle

\setcounter{page}{0}
\thispagestyle{empty}

\maketitle

\begin{abstract}
  The fundamental theorem of statistical learning states that for
  \emph{binary} classification problems, any Empirical Risk
  Minimization (ERM) learning rule has close to optimal sample
  complexity. In this paper we seek for a generic optimal learner for
  \emph{multiclass} prediction.  We start by proving a surprising
  result: a generic optimal multiclass learner must be
  \emph{improper}, namely, it must have the ability to output
  hypotheses which do not belong to the hypothesis class, even though
  it knows that all the labels are generated by some hypothesis from
  the class. In particular, no ERM learner is optimal. This brings
  back the fundmamental question of ``how to learn''? We give a
  complete answer to this question by giving a new analysis of the
  one-inclusion multiclass learner of \cite{rubinstein2006shifting}
  showing that its sample complexity is essentially optimal. Then,
  we turn to study the popular hypothesis class of generalized linear
  classifiers. We derive optimal learners that, unlike the one-inclusion algorithm, are
  computationally efficient. Furthermore, we show that the sample complexity of these 
  learners is better than the sample complexity of the ERM rule, thus settling in negative an open question due to
  \cite{Collins05}.
\end{abstract}

\newpage
\section{Introduction}

Multiclass classification is the problem of learning a classifier $h$
from a domain $\cx$ to a label space $\cy$, where $|\cy|>2$ and the
error of a prediction is measured by the probability that $h(x)$ is
not the correct label.  It is a basic problem in machine learning,
surfacing a variety of domains, including object recognition, speech
recognition, document categorization and many more.  Over the years,
multiclass classification has been subject to intense study, both
theoretical
\citep{Natarajan89b,Ben-DavidCeHaLo95,rubinstein2006shifting,DanielySaBeSh11,DanielySaSh12}
and practical
(e.g. \citep{ShalevKeSi04,Collins05,KeshetShSiCh05,TorralbaMuFr06}).
Many methods have been developed to tackle this problem, starting from
the the naive one-vs-all method, to more complex methods, such as
structured output prediction
\citep{Collins00,Collins02,LaffertyMcPer01,TaskarGuKo03,TsochantaridisHoJoAl04},
error correcting output codes \citep{DietterichBa95} and others.
These developments made it possible to handle a variety of multiclass
classification problems, including even problems that have a very
complex label space, that is structured and exponentially large
(e.g. speech recognition, OCR, and multiple object categorization).

Despite being very basic and natural, and despite these developments
and efforts, our theoretical understanding of multiclass
classification is still far from being satisfactory, in particular
relatively to our understanding of binary classification (i.e., when
$|\cy|=2$).  In this work, we focus on the sample complexity of
(distribution free) learning of hypothesis classes
$\ch\subseteq\cy^{\cx}$.  The two most fundamental questions are:
\begin{enumerate}
\item What is learnable? More quantitatively, what is the sample complexity of a given class $\ch$?
\item How to learn? In particular, is there a generic algorithm with optimal sample complexity?
\end{enumerate}
For binary classification problems, these two questions are
essentially solved (up to log-factors of the error and confidence
parameters $\epsilon$ and $\delta$): The fundamental result of
\cite{VapnikCh71} asserts that the $\mathrm{VC}$ dimension
characterizes the sample complexity, and that any Empirical Risk
Minimization (ERM) algorithm enjoys close-to-optimal sample
complexity.

In a recent surprising result, \cite{DanielySaBeSh11} have shown that
in multiclass classification there might be substantial gaps between
the sample complexity of different ERMs. We start by showing an even stronger ``peculiarity", discriminating binary
from multiclass classification. Recall that an algorithm is called {\em improper} if it might return a hypothesis that does
not belong to the learnt class. Traditionally, improper
learning has been applied to enable efficient
computations. It seems counter intuitive that computationally
unbounded learner would benefit from returning a hypothesis outside of
the learnt class. Surprisingly, we show that an optimal learning
algorithm {\em must} be improper! Namely, we show that there are
classes that are learnable {\em only} by an improper algorithm.
Pointing out that we actually
do not understand how to learn optimally, these results ``reopen" the
above two basic questions for multiclass classification.

In this paper we essentially resolve these two questions. We give a
new analysis of the multiclass one inclusion algorithm
(\cite{rubinstein2006shifting} based on \cite{HausslerLiWa88}, see
also \cite{simon2010one}), showing that it is optimal up to a constant
factor of $2$ in a transductive setting. This improves on the original
analysis, that yielded optimality only up to a factor of $\log(|\cy|)$
(which, as explained, might be quite large in several situations). By
showing reductions from transductive to inductive learning, we
consequently obtain an optimal learner in the PAC model, up to a
logarithmic factor of $\frac{1}{\delta}$ and $\frac{1}{\epsilon}$. The
analysis of the one inclusion algorithm results with a
characterization of the sample complexity of a class $\ch$ by a
sequence of numbers $\mu_{\ch}(m)$. Concretely, it follows that the
best possible guarantee on the error, after seeing $m$ examples, is
$\Theta\left(\frac{\mu_{\ch}(m)}{m}\right)$.

Comparing to binary classification, we should still strive for a
better characterization: We would like to have a characterization of
the sample complexity by a {\em single number} (i.e. some notion of
dimension) rather than a sequence.  Our analysis of the one inclusion
algorithm naturally leads to a new notion of dimension, of somewhat
different character than previously studied notions. We show that this
notion have certain advantages comparing to other previously studied
notions, and formulate a concrete combinatorial conjecture that, if
true, would lead to a crisper characterization of the sample
complexity.

Departing general theory, we turn our focus to investigate hypothesis
classes that are used in practice, in light of the above results and
the result of \cite{DanielySaBeSh11}. We consider classes of
multiclass linear classifiers that are learnt by several popular
learning paradigms, including multiclass SVM with kernels
\citep{CrammerSi01a}, structured output prediction
\citep{Collins00,Collins02,LaffertyMcPer01,TaskarGuKo03,TsochantaridisHoJoAl04},
and others.
Arguably, the two most natural questions in this context are: (i) is the ERM rule still sub-optimal even for such classes? and (ii) If yes, are there {\em efficient} optimal learnears for these classes?

Regarding the first question, we show that even though the sample
complexity of these classes is upper bounded in terms of the dimension
or the margin, there are sub-optimal ERMs whose sample complexity has additional
multiplicative factor that depends on the number of labels. This
settles in negative an open question due to \cite{Collins05}.
Regarding the second question above, as opposed to the one-inclusion
algorithm, which is in general inefficient, for linear classes we
derive computationally efficient learners (provided that the
hypotheses can be evaluated efficiently), that enjoy optimal sample
complexity.

\paragraph{Basic definitions:} 

Let $\cx$ be an instance space and $\cy$ a label space. To account for
margin-based classifiers as well, it would be convenient to allow
classifiers to return the label $\circleddash$ that will stand for
``don't know''. A classifier (or hypothesis) is a mapping $h : \cx \to
\left(\cy\cup\{\circleddash\}\right)$. A hypothesis class is a set of
classifiers, $\ch \subset
\left(\cy\cup\{\circleddash\}\right)^{\cx}$. The error of a classifier
with respect to a joint distribution over $\cx \times \cy$ is the
probability that $h(x) \neq y$. Throughout this paper, we mainly
consider learning in the realizable case, which means that there is
$h^* \in \ch$ which has zero error (extensions to agnostic learning are discussed in section \ref{sec:agnostic}). Therefore, we can focus on the
marginal distribution $\cd$ over $\cx$ and denote the error of a
classifier $h$ with respect to the realizing classifier $h^*$ as
$\Err_{\cd,h^*}(h):=\Pr_{x\sim\cd }\left(h(x)\ne h^*(x)\right)$.

A {\em learning algorithm} is a function $\ca$ that receives a
training set of $m$ instances, $S \in \cx^m$, together with their
labels according to $h^*$. We denote the restriction of $h^*$ to the
instances in $S$ by $h^*|_S$. The output of the algorithm $\ca$,
denoted $\ca(S,h^*|_S)$ is a classifier.  A learning algorithm is
\emph{proper} if it always outputs a hypothesis from $\ch$.  A learning
algorithm is an \emph{ERM learner} for the class $\ch$ if, for any
sample, it returns a function in $\ch$ that minimizes the empirical
error relative to any other function in $\ch$.  The \emph{(PAC) sample
  complexity} of a learning algorithm $\ca$ is the function
$m_{\ca,\ch}$ defined as follows: For every $\epsilon,\delta>0$,
$m_{\ca,\ch}(\epsilon,\delta)$ is the minimal integer such that for
every $m\ge m_{\ca,\ch}(\epsilon,\delta)$, every distribution $\cd$ on
$\cx$, and every target hypothesis $h^* \in \ch$,
$\Pr_{S\sim \cd^m}\left(\Err_{\cd,h^*}(\ca(S,h^*|_S))> \epsilon\right)\le\delta$.
Here and in subsequent definitions, we omit the subscript
$\ch$ when it is clear from context.  If no integer
satisfying the inequality above, define
$m_{\ca}(\epsilon,\delta)=\infty$. $\ch$ is learnable with $\ca$ if
for all $\epsilon$ and $\delta$ the sample complexity is finite. The
\emph{(PAC) sample complexity} of a class $\ch$ is
$m_{\pac,\ch}(\epsilon,\delta)=
\inf_{\ca}m_{\ca,\ch}(\epsilon,\delta)$,
where the infimum is taken over all learning algorithms.  The {\em ERM
  sample complexity} (a.k.a. the {\em uniform convergence sample
  complexity}) of $\ch$ is the sample complexity that can be
guaranteed for any ERM learner. It is defined by $
m_{\erm,\ch}(\epsilon,\delta)= \sup_{\ca \in
  \erm}m^a_{\ca,\ch}(\epsilon,\delta) $ where the supremum is taken
over all ERM learners for $\ch$. Clearly, we always have $m_{\pac}\le
m_{\erm}$.

We use $[m]$ to denote the set $\{1,\ldots,m\}$.  We treat vectors as
column vectors. We denote by $e_i\in\reals^d$ the $i$'th vector in the
standard basis of $\mathbb{R}^d$. We denote by $B^d$ the closed unit
ball in $\mathbb{R}^d$. We denote by $M_{d\times k}$ the space of real
matrices with $d$ rows and $k$ columns. For a matrix $X\in M_{d\times
  k}$ and $i\in [k]$, we denote by $X^i\in\mathbb R^d$ the $i$'th
column of $X$. Given a subset $A \subseteq \cx$, we define
$\ch|_A=\{h|_A:h\in\ch\}$.

\section{No optimal learner can be proper}\label{sec:no_imp}

Our first result shows that, surprisingly, any learning algorithm with
a close to optimal sample complexity must be improper.
\begin{theorem}\label{thm:counter_example}
  For every $1\le d\le \infty$ there exists a hypothesis class
  $\ch_d$, with $2^d+1$ labels such that:
\begin{itemize}
\item The PAC sample complexity of $\ch_d$ is
  $O\left(\frac{\log(1/\delta)}{\epsilon}\right)$.
\item The
  PAC sample complexity of any proper learning algorithm for $\ch_d$ is
  $\Omega\left(\frac{d+\log(1/\delta)}{\epsilon}\right)$.
\item In particular, $\ch_\infty$ is a learnable class that is not
  learnable by a proper algorithm.
\end{itemize}
\end{theorem}
A detailed proof is given in the appendix, and here we sketch the main idea
of the proof. 
Let $\cx$ be some finite set and let $\cy=2^{\cx}\cup\{*\}$. For every
$A\subseteq \cx$ define $h_A:\cx\to\cy$ by
$h_A(x)=\begin{cases}
A & x\in A
\\
* & \text{otherwise}
\end{cases}$.
Consider the hypothesis class $
\ch_{\cx,\mathrm{Cantor}}=\left\{h_A\mid A\subset\cx\right\}~.$
This class is due to \cite{DanielySaBeSh11} and we call it {\em
  the first Cantor class} due to the resemblance to the construction
used for proving the famous theorem of Cantor from set theory (e.g., \url{http://en.wikipedia.org/wiki/Cantor's_theorem}).
\cite{DanielySaBeSh11} employed this class to establish gaps between
the sample complexity of different ERM learners. 
In particular, they
have shown that there is an ERM learner with sample complexity  $\le\frac{\ln(1/\delta)}{\epsilon}$,
while there are other ERMs whose sample complexity is
$\Omega\left(\frac{|\cx|+\ln(1/\delta)}{\epsilon}\right)$. 

To show that no proper learner can be optimal, let $\cx_d$ be a set
consisting of $d$ elements and define the following subclass of
$\ch_{\cx_d,\mathrm{Cantor}}$: 
$
\ch_d=\left\{h_A\mid |A|=\left\lfloor\frac{d}{2}\right\rfloor\right\}$.
Since $\ch_d\subset\ch_{\cx_d,\mathrm{Cantor}}$, we can apply the
``good'' ERM learner described in \cite{DanielySaBeSh11} with respect
to the class $\ch_{\cx_d,\mathrm{Cantor}}$ and obtain an algorithm for
$\ch_d$ whose sample complexity is
$\le\frac{\ln(1/\delta)}{\epsilon}$. Note that this algorithm is
improper --- it might output a hypothesis from
$\ch_{\cx_d,\mathrm{Cantor}}$ which is not in $\ch_d$. As we show, no
proper algorithm is able to learn $\ch_d$ using
$o\left(\frac{d}{\epsilon}\right)$ examples. To understand the main
point in the proof, suppose that an adversary chooses $h_A\in\ch_d$
uniformly at random, and let the algorithm learn it, where the
distribution on $\cx_d$ is uniform on the complement of $A$, denoted
$A^c$.  Now, the error of every hypothesis $h_B\in\ch_d$ is
$\frac{|B\setminus A|}{d}$.  Therefore, to return a hypothesis with
small error, the algorithm must recover a set that is almost disjoint
from $A$, and therefore should recover $A$. However, if it sees only
$o(d)$ examples, all it knows is that some $o(d)$ elements in $\cx$ do
not belong to $A$.  It is not hard to be convinced that with this
little information, the probability that the algorithm will succeed is
negligible.

\section{An optimal learner for general classes}

In this section we describe and analyze a generic optimal learning algorithm.
We start with an algorithm for a transductive learning setting, in
which the algorithm observes $m-1$ labeled examples
and an additional unlabeled example, and it should output the
label of the unlabeled example. Later, in \secref{sec:transinduc} we
show a generic reduction from the transductive setting to the
usual inductive learning model (that is, the vanilla PAC model). 

Formally, in the transductive model, the algorithm observes a set of
$m$ unlabeled examples, $S \in \cx^m$, and then one of them is picked
uniformly at random, $x \sim U(S)$. The algorithm observes the labels
of all the examples but the chosen one, and should predict the label
of the chosen example. That is, the input of the algorithm, $\ca$, is
the set $S \in \cx^m$, and the restriction of some $h^* \in \ch$ to $S
\setminus x$, denoted $h^*|_{S \setminus x}$. The algorithm should
output $y \in \cy$. The error rate of a transductive algorithm $\ca$
is the function $\epsilon_{\ca,\ch} : \mathbb{N} \to [0,1]$ defined as
$
\epsilon_{\ca,\ch}(m) ~=~ \sup_{S \in \cx^m,h^* \in \ch} \, \left[ \Pr_{x \sim U(S)}\left(
\ca(S,h^*|_{S \setminus x}) \neq h^*(x)\right) \right]$.
The error rate of a class $\ch$ in the transductive model is defined
as $\epsilon_{\ch}(m) = \inf_{\ca} \epsilon_{\ca,\ch}(m)$,
where the infimum is over all transductive learning algorithms.

\subsection{The one-inclusion algorithm}
We next describe the one-inclusion transductive learning algorithm of
\cite{rubinstein2006shifting}.  Let $S = \{x_1,\ldots,x_m\}$ be an
unlabelled sample. For every $i \in [m]$ and $h \in \ch|_S$, let
$e_{i,h} \subset \ch|_S$ be all the hypotheses in $\ch|_S$ whose
restriction to $S \setminus \{x_i\}$ equals to $h|_{S \setminus
  \{x_i\}}$. That is, $h' \in e_{i,h}$ iff for all $j \neq i$ we have
$h'(x_j)=h(x_j)$.  Note that if $h' \in e_{i,h}$ then
$e_{i,h'}=e_{i,h}$.

Given
$(x_1,y_1),\ldots,(x_{i-1},y_{i-1}),(x_{i+1},y_{i+1}),\ldots,(x_m,y_m)$
let $h \in \ch|_S$ be some hypothesis for which $h(x_j)=y_j$ for all
$j \neq i$. We know that the target hypothesis can be any hypothesis
in $e_{i,h}$. Therefore, we can think on the transductive algorithm as
an algorithm that obtains some $e_{i,h}$ and should output one
hypothesis from $e_{i,h}$. Clearly, if $|e_{i,h}|=1$ we know that the
target hypothesis is $h$. But, what should the algorithm do when
$|e_{i,h}| > 1$ ?

The idea of the one-inclusion algorithm is to think on the collection
$E = \{ e_{i,h} \}_{i \in [m], \ch \in \ch|S}$ as a collection of
hyperedges of a hypergraph $G = (V,E)$. Recall that in a hypergraph,
$V$ is some set of vertices and each hyperedge $e \in E$ is some
subset of $V$. In our case, the vertex set is $V = \ch|_S$. This hypergraph
is called the one-inclusion hypergraph. Note that if $|e|=2$ for every
$e \in E$ we obtain the usual definition of a graph. In such a case,
an \emph{orientation} of an undirected edge $e = \{v_1,v_2\}$ is
picking one of the vertices (e.g. $v_1$) to be the ``head'' of the
edge. Similarly, an orientation of a hyperedge is choosing one $v \in
e$ to be the ``head'' of the hyperedge. And, an orientation of the
entire hypergraph is a function $f : E \to V$ such that for all $e \in
E$ we have that $f(e) \in e$.

Getting back to our transductive learning task, it is easy to see that
any (deterministic) transductive learning algorithm is equivalent to an orientation
function $f : E \to V$ of the one-inclusion hypergraph. The error rate
of such an algorithm, assuming the target function is $h^* \in \ch|_S$, is 
\begin{equation} \label{eqn:outDegErrRelation}
\Pr_{i \sim U([m])} [ f(e_{i,h^*}) \neq h^* ] = 
\frac{1}{m} \sum_{i=1}^m 1[f(e_{i,h^*}) \neq h^*] =
\frac{|\{e \in
  E : h^* \in e \land f(e) \neq h^*\}|}{m} ~.
\end{equation}
The quantity $|\{e \in E : h^* \in e \land f(e) \neq h^*\}|$ is
called the \emph{out-degree} of the vertex $h^*$ and denoted
$d^+(h^*)$. It follows that the error rate of an orientation $f$ is
$\max_{h^* \in \ch|_S} \, \frac{d^+(h^*)}{m}$.
It follows that the best deterministic transductive
algorithm should find an orientation of the hypergraph that
minimizes the maximal out degree. This leads to the one-inclusion
algorithm. 

\begin{algorithm}[th] 
\caption{Multiclass one inclusion algorithm for $\ch\subset\cy^{\cx}$} \label{algo:one_inc}
\begin{algorithmic}[1]
\STATE {\bf Input:} unlabeled examples $S = (x_1,\ldots,x_m)$, labels $(y_1,\ldots,y_{i-1},y_{i+1},\ldots,y_m)$
\STATE Define the one-inclusion graph $G = (V,E)$ where $V = \ch|_S$
and $E = \{ e_{j,h} \}_{j \in [m],h \in V}$
\STATE Find orientation $f : E \to V$ that minimizes the maximal
out-degree of $G$
\STATE Let $h \in V$ be s.t. $h(x_j)=y_j$ for all $j \neq i$, and let $\hat{h} = f(e_{i,h})$
\STATE {\bf Output:} predict $\hat{h}(x_i)$
\end{algorithmic}
\end{algorithm}

\subsection{Analysis}

The main result of this section is a new analysis of the one inclusion
algorithm, showing its optimality in the transductive model, up to a
constant factor of $1/2$. In the next subsection we deal with the PAC model.

To state our results, we need a few definitions.  Let
$G=(V,E)$ be a hypergraph. Throughout, we only consider
hypergraphs for which $E$ is an antichain (i.e., there are no
$e_1,e_2\in E$ such that $e_1$ is strictly contained in
$e_2$). 
Given $U \subseteq V$, define the \emph{induced} hypergraph,
$G[U]$, as the hypergraph whose vertex set is $U$ and whose edge
set is all sets $e\subseteq U$ such that $e=U\cap e'$ for some $e'\in
E$, $|e|\ge 2$, and $e$ is maximal w.r.t. these conditions.

The \emph{degree} of a vertex $v$ in a hypergraph $G = (V,E)$ is the
number of hyperedges, $e \in E$, such that $|e| \ge 2$ and $v \in
e$. The {\em average degree} of $G$ is $d(G)=\frac{1}{|V|}\sum_{v\in
  V}d(v)$. The {\em maximal average degree} of $G$ is
$\md(G)=\max_{U\subseteq V : |U|<\infty}d(G[U])$.
For a hypothesis class $\ch$ define 
\[
\mu_{\ch}(m)=\max\{\md (G(\ch|_S))\mid S \in \cx^m\}~,
\]
where $G(\ch|_S)$ is the one-inclusion hypergraph defined in Algorithm~\ref{algo:one_inc}. 

\begin{theorem}\label{thm:main_opt_alg}
For every class $\ch$, $\frac{1}{2} \frac{\mu_\ch(m)}{m} \le \epsilon_{\ch}(m)\le
\frac{\mu_\ch(m)}{m}$.
\end{theorem}
\begin{proof}
To prove the upper bound, recall that the one inclusion algorithm
uses an orientation of the one-inclusion hypergraph that minimizes the
maximal out-degree, and recall that in \eqref{eqn:outDegErrRelation}
we have shown that the error rate of an orientation function is upper
bounded by the maximal out-degree over $m$. Therefore, the proof of
the upper bound of the theorem follows directly from the following
lemma:
\begin{lemma}\label{lem:out_deg}
  Let $G=(V,E)$ be a hypergraph with maximal average degree $d$. Then,
  there exists an orientation of $G$ with maximal out-degree of at most $d$.
\end{lemma}
The proof of the lemma is given in the appendix.

While the above proof of the upper bound is close in spirit to the
arguments used by \cite{HausslerLiWa88} and
\cite{rubinstein2006shifting}, the proof of the lower bound relies on
a new argument.  As opposed to \cite{rubinstein2006shifting} who lower
bounded $\epsilon_{\ch}(m)$ using the Natarajan dimension, we give a
direct analysis.

Let $S\in \cx^m$ be a set such that $\md(G(\ch|_{S}))=\mu_{\ch}(m)$. For simplicity we assume that $|S|=m$ (i.e., $S$ does not contain multiple elements). Since $\md(G(\ch|_{S}))=\mu_{\ch}(m)$, there is finite $\cf\subset\cg$ with $d(G(\cf|_{S}))=\mu_{\ch}(m)$.
Consider the following scenario. Suppose that $h^*\in \cf|_{S}$ is chosen uniformly at random, and in addition, a point $x\in S$ is also chosen uniformly at random.
Now, suppose that a learner $\ca$ is given the sample $S$ with all points labelled by $h^*$ except $x$ that is unlabelled. It is enough to show that the probability that $\ca$ errs is $\ge \frac{\mu_{\ch}(m)}{2m}$.

Denote by $U$ the event that $x$ correspond to an edge in $G(\cf|_{S})$ coming out of $h^*$. Given $U$, the value of $h^*(x)$, given what the algorithm sees, is distributed uniformly in the set $\{h(x)\mid h\in\cf\text{ and }h|_{S\setminus\{x\}}=h^*|_{S\setminus\{x\}}\}$. Since this set consists of at least two elements, given $U$, the algorithm errs with probability $\ge\frac{1}{2}$.

It is therefore enough to prove that $\Pr(U)\ge \frac{\mu_\ch(m)}{m}$. Indeed, given $h^*$, the probability that $x$ corresponds to an edge coming out of $h^*$ is exactly the degree of $h^*$ over $m$. Therefore, the probability that $x$ corresponds to an edge coming out of a randomly chosen $h^*$ is the average degree of $G(\cf|_S)$ over $m$, i.e., $\frac{\mu_\ch(m)}{m}$.
\end{proof}

\subsection{PAC optimality: from transductive to inductive learning} \label{sec:transinduc}
In the previous section we have analyzed the optimal error rate of
learning in the transductive learning. We now turn to the inductive
PAC model. 
By a simple reduction from inductive to transductive learning, we will show that a variant of the one-inclusion algorithm is essentially optimal in the PAC model. 

First, any transductive algorithm $\ca$ can be naturally interpreted as an inductive algorithm, which we denote by $\ca^i$.
Specifically, $\ca^i$ returns, after seeing the sample $S=\{(x_i,y_i)\}_{i=1}^{m-1}$, the hypothesis $h:\cx\to\cy$ such that $h(x)$ is the label $\ca$ would have predicted for $x$ after seeing the labelled sample $S$.

It holds that (see the appendix) the (worst case) expectation of the error of the hypothesis returned by $\ca^{i}$ operating on $m$ points sample, is the same, up to a factor of $e$ to $\epsilon_{\ca}(m)$. Using this fact and a simple amplification argument, it is not hard to show that a variant of the one-inclusion algorithm is essentially optimal in the PAC model.

Namely, we consider the algorithm $\overline{\ci}$ that splits the sample into $2\log(1/\delta)$ parts, run the one inclusion algorithm on $\log(1/\delta)$ different parts to obtain $\log(1/\delta)$ candidate hypotheses, and finally chooses the best one, by validation on the remaining points. As the following theorem (whose proof is given in the appendix) shows, $\overline{\ci}$ is optimal up to a factor of $O\left(\log\left(\frac{1}{\delta}\right)\log\left(\frac{1}{\epsilon}\right)\right)$ in the PAC model, in the following sense:
\begin{theorem}\label{thm:main_opt_alg_pac}
For some $c>0$, and every class $\ch$,
$m_{\overline{\ci},\ch}(\epsilon,\delta)
\le
m_{\mathrm{PAC},\ch}\left(c\epsilon,\delta\right)\cdot \frac{1}{c}\log(1/\delta)\log(1/\epsilon)$.
\end{theorem}

\section{Efficient optimal learning and gaps for linear classes}

In this section we study the family of linear hypothesis classes. This
family is widely used in practice and received a lot of attention in
the literature---see for example
\cite{CrammerSi01a,Collins00,Collins02,LaffertyMcPer01,TaskarGuKo03,TsochantaridisHoJoAl04}.
We show that, rather surprisingly, even for such simple classes,
there can be gaps between the ERM sample complexity and the PAC sample
complexity. This settles in negative an open question raised by
\cite{Collins05}. We also derive computationally efficient optimal
learners for linear classes, based on the concept of compression
schemes. This is in contrast to the one-inclusion algorithm from the
previous section, which in general is inefficient.  Due to
the lack of space, most proofs are deferred to the appendix.

\subsection{Linear hypothesis classes}

We first define the various hypothesis classes of multiclass linear
classifiers that we study. All of these classes depend on a
class-specific feature mapping, $\Psi:\cx\times \cy\to \mathbb
\reals^d$.  We will provide several examples of feature mappings that are
widely used in practice.

\subsubsection{Dimension based linear classifiers (denoted $\ch_{\Psi}$)}
For $w\in\mathbb \reals^d$ and $x\in\cx$, define the multiclass predictor
$h_w(x) = \argmax_{y \in \cy} \inner{w,\Psi(x,y)}$.
In case of a tie, $h_w(x)$ is assumed to be the ``don't know label'', $\circleddash$.
The corresponding hypothesis class is defined as $\ch_{\Psi}=\{h_w\mid w\in\mathbb \reals^d\}$.
\begin{example}[multivector construction]\label{examp:multi_no_margin}
  If the labels are unstructured, a canonical choice of $\Psi$
  is the so called {\em multivector construction}. Here, $\cy=[k]$,
  $\cx=\mathbb \reals^d$ and $\Psi:\cx\times\cy\to\mathbb \reals^{dk}$ is
  defined as follows: $\Psi(x,y)$ is the $d\times k$
  matrix whose $y$'th column is $x$, while the rest are $0$. In this
  case, every classifier corresponds to a matrix $W$, and the
  prediction on an instance $x\in\mathbb \reals^d$ is the index of the
  column that maximizes the inner product with $x$.
\end{example}

\subsubsection{Large margin linear classifiers (denoted $\ch_{\Psi,R}$)}
The second kind of hypothesis class induced by $\Psi$ is margin
based. Here, we assume that the range of $\Psi$ is contained in the
unit ball of $\reals^d$. Every vector $w\in \mathbb \reals^d$ defines a function
$h_{w}:\cx\to\left(\cy\cup\{\circleddash\}\right)$ by
\[
\forall x\in\cx,\;\;h_w(x)=
\begin{cases}
y & \text{if}\inner{w,\Psi(x,y)-\Psi(x,y')}\ge 1\text{ for every }y'\ne y
\\
\circleddash & \text{if no such $y$ exists}
\end{cases}
\]
The class of {\em linear classifiers of complexity $R>0$ induced by $\Psi$} is
$\ch_{\Psi,R}=\left\{h_w\mid \|w\|^2\le R\right\}$.

\begin{example}[multivector construction with margin]\label{examp:multi_with_margin}
The margin based analogue to example \ref{examp:multi_no_margin} is
defined similarly. This class is the class
that is learnt by multiclass SVM.
\end{example}

\subsubsection{The classes $\ch_{d,t,q}$ and $\ch_{d,t,q,R}$ for structured output prediction}
Next we consider an embedding $\Psi$ that is specialized and used in
classification tasks where the number of possible labels is
exponentially large, but the labels are structured
(e.g. \cite{TaskarGuKo03}). For example, in speech recognition, the
label space might me the collection of all sequences of $\le 20$
English words.

To motivate the definition, consider the case that we are to recognize
a $t$-letter word appearing in an image. Let $q$ be the size of the
alphabet.  The set of possible labels is naturally associated with
$[q]^t$. A popular method to tackle this task (see for example
\cite{TaskarGuKo03}) is the following: The image is broken into $t$
parts, each of which contains a single letter. Each letter is
represented as a vector in $\mathbb \reals^d$. Thus, each image is
represented as a matrix in $M_{d\times t}$. To devise a linear
hypothesis class to this problem, we should specify a mapping
$\Psi:M_{d\times t}\times [q]^t\to\mathbb \reals^n$ for some
$n$. Given $X\in M_{d\times t}$ and $y\in [q]^t$, $\Psi(X,y)$ will be
a pair $(\Psi_1(X,y),\Psi_2(X,y))$.  The mapping $\Psi_1$ allows the
classifiers to take into account the shape of the letters appearing in
the different $t$ parts the word was broken into. The mapping $\Psi_2$
allows the classifiers to take into account the structure of the
language (e.g. the fact that the letter ``u" usually appears after the
letter ``q").  $\Psi_1(X,y)\in M_{d\times q}$ is the matrix whose
$j$'th column is the sum of the columns $X^i$ with $y_i=j$ (in other
words, the $j$'th column is the sum of the letters in the image that
are predicted to be $j$ by $y$). $\Psi_2(X,y)\in M_{q,q}$ will be the
matrix with $1$ in the $(i,j)$ entry if the letter $j$ appears after
the letter $i$ somewhere in the word $y$, and $0$ in all other
entries. Even though the number of labels is exponential in $t$, this
class (in the realizable case) can be learnt in time polynomial in
$d,t$ and $q$ (see \cite{Collins05}). 

We will show gaps in the performance of different ERMs for the class
$\ch_{\Psi}$. If fact, we will prove a slightly stronger result. We
will consider the class $\ch_{\Psi_1}$, that we will denote by
$\ch_{d,t,q}$. It is easy to see that $\ch_{\Psi_1}$ can be realized
by $\ch_{\Psi}$. Therefore, any lower bound for $\ch_{\Psi_1}$
automatically lower bounds also $\ch_{\Psi}$. As for upper bounds, as
long as $q=O(d)$, the upper bounds we show are the same for
$\ch_{\Psi}$ and $\ch_{\Psi_1}$. To summarize, the gaps we show for
$\ch_{\Psi_1}$ automatically (as long as $q=O(d)$) hold for
$\ch_{\Psi}$ as well. 

Finally, we define a margin-based analogue to $\ch_{d,t,q}$. The
instance space is $(B^d)^t$, and we treat each $X\in (B^d)^t$ as
a matrix with $t$ columns, each of which is a vector in $B^d$. The
labels are $[q]^t$. Define $\Psi:(B^d)^t\times [q]^k\to M_{d\times
  q}$ as follows: for $X\in(B^d)^t$ and $y\in [q]^t$,
$\Psi(X,y)$ is the matrix whose $j$'th column is $\frac{1}{q}$ of the
average of all columns $X^i$ such that $y_i=j$. Note that the range of
$\Psi$ is contained in the unit ball. For $R>0$, define
$\ch_{d,t,q,R}:=\ch_{\Psi,R}$.

\subsection{Results}
We begin with linear predictors without margin. The first
part of the following theorem asserts that for every
$\Psi:\cx\times\cy\to\mathbb R^d$ there is some algorithm that
learns $\ch_{\Psi}$ with sample complexity
$O\left(\frac{d\log(1/\epsilon)+\log(1/\delta)}{\epsilon}\right)$. The
second part of the theorem shows that in several cases (i.e., for some
$\Psi$'s), this algorithm outperforms other ERMs, by a factor of
$\log(|\cy|)$.
\begin{theorem}\label{thm:gen_non_margin}
\begin{list}{\labelitemi}{\leftmargin=0.5em\itemsep=0.2pt\parsep=-1pt}
\item For every $\Psi:\cx\times\cy\to\mathbb R^d$, the PAC sample
  complexity of $\ch_{\Psi}$ is
  $O\left(\frac{d\log(1/\epsilon)+\log(1/\delta)}{\epsilon}\right)$,
  and is achievable by a new efficient\footnote{Assuming we have an appropriate separation oracle.} compression scheme.
\item For every $\cy$ and $d>0$, there is some $\Psi:\cx\times\cy\to\mathbb R^d$ for which the ERM sample complexity of $\ch_{\Psi}$ is $\Omega\left(\frac{d\log(|\cy|)+\log(1/\delta)}{\epsilon}\right)$.
\end{list}
\end{theorem}
To put the result in the relevant context, it was known
(e.g. \cite{DanielySaBeSh11}) that the sample complexity of every ERM
for this class is
$O\left(\frac{d\log(|\cy|)\log(1/\epsilon)+\log(1/\delta)}{\epsilon}\right)$. In
particular, the second part of the theorem is tight, up to the
logarithmic dependence over $\frac{1}{\epsilon}$. However, it was not
known whether the factor of $\log(|\cy|)$ for general ERM is
necessary. The second part of the theorem shows that this factor is
indeed necessary.

As to the tightness of the first part, for certain embeddings,
including the multivector construction (example
\ref{examp:multi_no_margin}), a lower bound of
$\Omega\left(\frac{d+\log(1/\delta)}{\epsilon}\right)$ is known for
every algorithm. Hence, the first part of the
theorem is also tight up to the logarithmic dependence over
$\frac{1}{\epsilon}$.

Our second theorem for linear classes is analogous to theorem
\ref{thm:gen_non_margin} for margin based classes.  The first part
shows that for every $\Psi:\cx\times\cy\to B^d$ there is some
algorithm that learns $\ch_{\Psi,R}$ with sample complexity
$O\left(\frac{R\log(1/\epsilon)+\log(1/\delta)}{\epsilon}\right)$. The
second part of the theorem shows that in several cases, the above
algorithm outperforms other ERMs, by a factor of $\log(|\cy|)$.
\begin{theorem}\label{thm:gen_margin}
\begin{list}{\labelitemi}{\leftmargin=0.5em\itemsep=0.2pt\parsep=-5pt}
\item For every $\Psi:\cx\times\cy\to B^d$ and $R>0$, the PAC sample complexity of $\ch_{\Psi,R}$ is $O\left(\frac{R\log(1/\epsilon)+\log(1/\delta)}{\epsilon}\right)$.
\item For every $\cy$ and $R>0$, there is some $\Psi:\cx\times\cy\to B^d$ for\footnote{Here, $d$ can be taken to be polynomial in $R$ and $\log(|\cy|)$.} which the ERM sample complexity of $\ch_{\Psi,R}$ is $\Omega\left(\frac{R\log(|\cy|)+\log(1/\delta)}{\epsilon}\right)$.
\end{list}
\end{theorem}
The first part of the theorem is not new.  An algorithm that achieves
this bound is the perceptron.  It was known (e.g. \cite{Collins05})
that the sample complexity of every ERM for this class is
$O\left(\frac{R\log(|\cy|/\epsilon)+\log(1/\delta)}{\epsilon}\right)$. In
particular, the second part of the theorem is tight, up to the
logarithmic dependence over $\frac{1}{\epsilon}$. However, it was not
known whether the gap is real: In \citep{Collins05}, it was left as an
open question to show whether the perceptron's bound holds for every
ERM. The second part of the theorem answers this open question in
negative.  Regarding lower bounds, as in the case of $\ch_{\Psi}$, for
certain embeddings, including the multivector construction with margin
(example \ref{examp:multi_no_margin}), a lower bound of
$\Omega\left(\frac{R+\log(1/\delta)}{\epsilon}\right)$ is known and
valid for every learning algorithm. In particular, the first part of
the theorem is also tight up to the logarithmic dependence over
$\frac{1}{\epsilon}$.

An additional result that we report on shows that, for every
$\Psi:\cx\times\cy\to\mathbb R^d$, the Natarajan dimension of
$\ch_{\Psi}$ is at most $d$ (the definition of the Natarajan dimension
is recalled in the appendix). This strengthens the result of
\citep{DanielySaBeSh11} who showed that it is bounded by
$O(d\log(d))$. It is known (e.g. \cite{DanielySaSh12}) that for the
multivector construction (example \ref{examp:multi_no_margin}), in
which the dimension of the range of $\Psi$ is $dk$, the Natarajan
dimension is lower bounded by $(d-1)(k-1)$. Therefore, the theorem is
tight up to a factor of $1+o(1)$.
\begin{theorem}\label{thm:main_nat_dim}
For every $\Psi:\cx\times\cy\to\mathbb R^d$, $\Ndim (\ch_\Psi) \le d$.
\end{theorem}
Next, we give analogs to theorems \ref{thm:gen_non_margin} and
\ref{thm:gen_margin} for the structured output classes $\ch_{d,k}$ and
$\ch_{d,k,R}$. These theorems show that the phenomenon of gaps between
different ERMs, as reported in \citep{DanielySaBeSh11}, happens also
in hypothesis classes that are used in practice.
\begin{theorem}\label{thm:struc_non_margin}
\begin{list}{\labelitemi}{\leftmargin=0.5em\itemsep=0.2pt\parsep=-5pt}
\item For every $d,t,q>0$, the PAC sample complexity of $\ch_{d,t,q}$ is $O\left(\frac{dq\log(1/\epsilon)+\log(1/\delta)}{\epsilon}\right)$.
\item For every $d,t,q>0$ the ERM sample complexity of $\ch_{d,t,q}$ is $\Omega\left(\frac{dq\log(t)+\log(1/\delta)}{\epsilon}\right)$.
\end{list}
\end{theorem}
\vspace{-0.5cm}
\begin{theorem}\label{thm:struc_margin}
\begin{list}{\labelitemi}{\leftmargin=0.5em\itemsep=0.2pt\parsep=-5pt}
\item For every $d,t,q,R>0$, the PAC sample complexity of $\ch_{d,t,q,R}$ is $O\left(\frac{R\log(1/\epsilon)+\log(1/\delta)}{\epsilon}\right)$.
\item For every $t,q,R>0$ and $d\ge (t+1)R$, the ERM sample complexity of $\ch_{d,t,q,R}$ is $\Omega\left(\frac{R\log(t)+\log(1/\delta)}{\epsilon}\right)$.
\end{list}
\end{theorem}
The first parts of theorems \ref{thm:struc_non_margin} and
\ref{thm:struc_margin} are direct consequences of theorems
\ref{thm:gen_non_margin} and \ref{thm:gen_margin}. These results are
also tight up to the logarithmic dependence over
$\frac{1}{\epsilon}$. The second parts of the
theorems do not follow from theorems \ref{thm:gen_non_margin} and
\ref{thm:gen_margin}. Regarding the tightness of the second part, the
best known upper bounds for the ERM sample complexity of $\ch_{d,t,q}$
and $\ch_{d,t,q,R}$ are
$O\left(\frac{dqt\log(\frac{1}{\epsilon})+\log(1/\delta)}{\epsilon}\right)$
and
$O\left(\frac{Rt\log(\frac{1}{\epsilon})+\log(1/\delta)}{\epsilon}\right)$
respectively.  Closing the gap between these upper bounds and the
lower bounds of theorems \ref{thm:struc_non_margin} and
\ref{thm:struc_margin} is left as an open question.

\subsection{The compression-based optimal learners} \label{subsec:compressionDim}
Each of the theorems \ref{thm:gen_non_margin}, \ref{thm:gen_margin}, \ref{thm:struc_non_margin} and \ref{thm:struc_margin} are composed of two statements. The first claims that some algorithm have a certain sample complexity, while the second claims that there exists an ERM whose sample complexity is worse than the sample complexity of the algorithm from the first part. As explained in this subsection, the first parts of these theorems are established by devising (efficient) compression schemes. In the next subsection we will elaborate on the proof of the second parts (the lower bounds on specific ERMs). Unfortunately, due to lack of space, we must be very brief.

We now show that for linear classes, it is possible to derive optimal
learners which are also computationally efficient. For the case of
margin-based classes, this result is not new --- an efficient
algorithm based on the multiclass perceptron has been proposed in
\cite{Collins02}. For completeness, we briefly survey this approach in
the appendix. For dimension based linear classes, we give a new
efficient algorithm. 

The algorithm relies on compression based generalization bounds (see
\thmref{thm:compression} in the appendix). Based on this theorem,
it is enough to show that for every $\Psi:\cx\times\cy\to\mathbb R^d$,
$\ch_\Psi$ has a compression scheme of size $d$. We consider the
following compression scheme. Given a realizable sample
$(x_1,y_1),\ldots,(x_m,y_m)$, let $Z\subseteq \mathbb R^d$ be the set
of all vectors of the form $\Psi(x_i,y_i)-\Psi(x_i,y)$ for $y\ne
y_i$. Let $w$ be the vector of minimal norm in the convex hull of $Z$,
  $\conv(Z)$. Note that by the convexity of $\conv(Z)$, $w$ is unique
  and can be found efficiently
  using a convex optimization procedure. Represent $w$ as a convex combination of $d$ vectors from
  $Z$. This is possible since, by claim \ref{claim:comp_1} below,
  $0\not\in \conv(Z)$. Therefore, $w$ is in the boundary of the
  polytope $\conv(Z)$. Thus, $w$ lies in a convex polytope whose
  dimension is $\le d-1$, and is the convex hull of points from
  $Z$. Therefore, by Caratheodory's theorem (and using its efficient
  constructive proof), $w$ is a convex
  combination of $\le d$ points from $Z$. Output the examples in the sample that correspond to the vectors in the above convex combination. If there are less than $d$ such examples, arbitrarily output more examples.

The {\em De-Compression} procedure is as follows. Given
$(x_1,y_1),\ldots,(x_d,y_d)$, let $Z'\subseteq \mathbb R^d$ be the set
of all vectors of the form $\Psi(x_i,y_i)-\Psi(x_i,y)$ for $y\ne
y_i$. Then, output the minimal norm vector $w\in \conv(Z')$.

In the appendix (\secref{subsec:compressionDimValidity}) we show that
this is indeed a valid compression scheme, that is, if we start with a
realizable sample $(x_1,y_1),\ldots,(x_m,y_m)$, compress it, and then
de-compress it, we are left with a hypothesis that makes no errors on
the original sample.

\subsection{Lower bounds for specific ERMs}
Next, we explain how we prove the second parts of theorems
\ref{thm:gen_non_margin}, \ref{thm:gen_margin},
\ref{thm:struc_non_margin} and \ref{thm:struc_margin}. For theorems
\ref{thm:gen_non_margin} and \ref{thm:gen_margin}, the idea is to
start with the first Cantor class (introduced in section
\ref{sec:no_imp}) and by a geometric construction, realize it by a
linear class. This realization enables us to extend the ``bad ERM" for
the first Cantor class, to a ``bad ERM" for that linear class. The
idea behind the lower bounds of theorems \ref{thm:struc_non_margin}
and \ref{thm:struc_margin} is similar, but technically more
involved. Instead of the first Cantor class, we introduce a new
discrete class, {\em the second Cantor class}, which may be of
independent interest. This class, which can be viewed as a dual to the first
Cantor class, is defined as follows.  Let $\tilde{\cy}$ be some
non-empty finite set. Let $\cx=2^{\tilde{\cy}}$ and let
$\cy=\tilde{\cy} \cup\{*\}$. For every $y\in \tilde{\cy}$ define a
function $h_y:\cx\to \cy$ by
$h_y(A)=\begin{cases}
y & y\in A
\\
* &\text{otherwise}
\end{cases}$.
Also, let $h_{*}:\cx\to\cy$ be the constant function $*$. Finally, let $\ch_{\cy,\mathrm{Cantor}}=\{h_y\mid y\in\cy\}$.
In section \ref{sec:disc_classes} we show that the graph dimension (see a definition in the appendix) of  $\ch_{\cy,\mathrm{Cantor}}$ is $\Theta(\log(|\cy|))$. 
The analysis of the graph dimension of this class is more involved than the first Cantor class: by a probabilistic argument, we show that a random choice of $\Omega\left(\log(|\cy |)\right)$ points from $\cx$ is shattered with positive probability.  We show also (see section \ref{sec:disc_classes}) that the PAC sample complexity of $\ch_{\cy,\mathrm{Cantor}}$ is $\le \frac{\log(1/\delta)}{\epsilon}$. Since the graph dimension characterizes the ERM sample complexity (see the appendix), this class provides another example of a hypothesis class with gaps between ERM and PAC learnability.


\section{A new dimension}
Consider again the question of characterizing the sample complexity of
learning a class $\ch$. \thmref{thm:main_opt_alg} shows that the
sample complexity of a class $\ch$ is characterized by the sequence of
densities $\mu_\ch(m)$.  A better characterization would be a notion
of dimension that assigns a single number, $\dim(\ch)$, that controls
the growth of $\mu_{\ch}(m)$, and consequently, the sample complexity
of learning $\ch$.  To reach a plausible generalization, let us return
for a moment to binary classification, and examine the relationships
between the VC dimension and the sequence $\mu_{\ch}(m)$. It is not
hard to see that
\begin{itemize}
\item The VC dimension of $\ch$ is the maximal number $d$ such that $\mu_{\ch}(d)=d$.
\end{itemize}
Moreover, a beautiful result of \cite{HausslerLiWa88} shows that 
\begin{itemize}
\item If $|\cy|=2$, then $\VCdim(\ch)\le \mu_\ch(m)\le 2\VCdim(\ch)$
for every $m\ge \VCdim(\ch)$.
\end{itemize}
These definition and theorem naturally suggest a generalization to multiclass classification:
\begin{definition}
The {\em dimension}, $\dim (\ch)$, of the class $\ch\subset\cy^{\cx}$ is the maximal number $d$ such that $\mu_\ch(d)=d$.
\end{definition}
\begin{conjecture}\label{conj:main}
There exists a constant $C>0$ such that for every $\ch$ and $m\ge \dim(\ch)$, 
$\dim(\ch)\le \mu_\ch(m)\le C\cdot\dim(\ch)~.$
Consequently, by \thmref{thm:main_opt_alg},
\[
\epsilon_{\ch}(m)=\Theta\left( \frac{\dim(\ch)}{m}\right)
~~\textrm{and}~~
\Omega\left(\frac{\dim(\ch)+\log\left(\frac{1}{\delta}\right)}{\epsilon}\right)\le m_{\ch}(\epsilon,\delta)\le O\left(\frac{\dim(\ch)\log\left(\frac{1}{\delta}\right)}{\epsilon}\right)
\]
\end{conjecture}
For concreteness, we give an equivalent definition of $\dim(\ch)$ and a formulation of conjecture \ref{conj:main} that are somewhat simpler, and do not involve the sequence $\mu_{\ch}(m)$
\begin{definition}
Let $\ch\subset\cy^{\cx}$. We say that $A\subset\cx$ is {\em 
shattered} by $\ch$ is there exists a finite $\cf\subset \ch$ such that for every $x\in A$ 
and $f \in\cf$ there is $g\in\cf$ such that $g(x)\ne f(x)$ and $g|_{A\setminus\{x\}}=f|_{A\setminus\{x\}}$. The {\em dimension} of $\ch$ is the maximal cardinality of a shattered set.
\end{definition}
Recall that the {\em degree} (w.r.t. $\ch\subset \cy^{\cx}$) of $f\in\ch$ is the number of points $x\in\cx$ for which there exists $g\in \ch$ that disagree with $f$ only on $x$.
We denote the {\em average degree} of $\ch$ by $d(\ch)$.
\begin{conjecture}
There exists $C>0$ such that for every finite $\ch$, $d(\ch)\le C\cdot \dim(\ch)$.
\end{conjecture}
By combination of theorems \ref{thm:main_opt_alg} and \cite{rubinstein2006shifting}, a weaker version of conjecture \ref{conj:main} is true. Namely, that for some absolute constant $C>0$
\begin{equation}\label{eq:ne_dim}
\dim(\ch)\le \mu_\ch(m)\le C\cdot\log(|\cy|)\cdot\dim(\ch)~.
\end{equation}
In addition, it is not hard to see that the new dimension is bounded
between the Natarajan and Graph dimensions, $
\Ndim(\ch)\le \dim(\ch)\le \Gdim(\ch)$.
For the purpose of characterizing the sample complexity, this
inequality is appealing for two reasons. First, it is known
\citep{DanielySaBeSh11} that the graph dimension does not characterize
the sample complexity, since it can be substantially larger than the
sample complexity in several cases. Therefore, any notion of dimension
that do characterize the sample complexity must be upper bounded by
the graph dimension.  As for the Natarajan dimension, it is known to
lower bound the sample complexity. By \thmref{thm:main_opt_alg}
and equation (\ref{eq:ne_dim}), the new dimension also lower bounds
the sample complexity. Therefore, the left inequality shows that the
new dimension always provides a lower bound that is at least as good
as the Natarajan dimension's lower bound.


\bibliography{bib}

\appendix

\section{Agnostic learning and further directions}\label{sec:agnostic}
In this work we focused on learning in the realizable setting.  For
general hypothesis classes, it is left as an open question to find an
optimal algorithm for the agnostic setting.  However, for linear
classes, our upper bounds are attained by compression
schemes. Therefore, as indicated by \thmref{thm:compression}, our
results can be extended to the agnostic setting, yielding algorithms
for $\ch_{\Psi}$ and $\ch_{\Psi,R}$ whose sample complexity is
$O\left(\frac{d\log(d/\epsilon)+\log(1/\delta)}{\epsilon^2}\right)$
and
$O\left(\frac{R\log(R/\epsilon)+\log(1/\delta)}{\epsilon^2}\right)$
respectively. We note that these upper bounds are optimal, up to the
factors of $\log(d/\epsilon)$ and $\log(R/\epsilon)$.  Our lower
bounds clearly hold for agnostic learning (this is true for any lower
bound on the realizable case). Yet, we would be excited to see better
lower bounds for the agnostic setting. Specifically, are there classes
$\ch\subset\cy^{\cx}$ of Natarajan dimension $d$ with ERMs whose
agnostic sample complexity is
$\Omega\left(\frac{d\log(|\cy|)}{\epsilon^2}\right)$?

Except extensions to the agnostic settings, the current work suggests several more directions for further research. First, it would be very interesting to go beyond multiclass classification, and to devise generic optimal algorithms for other families of learning problems. Second, as noted before, naive implementation of the one-inclusion algorithm is prohibitively inefficient. Yet, we still believe that the ideas behind the one-inclusion algorithm might lead to better efficient algorithms. 
In particular,
it might be possible to derive efficient algorithms based on the principles behind the one-inclusion algorithm, and maybe even give an efficient implementation of the one-inclusion algorithm for concrete hypothesis classes.

\section{Background}

\subsection{The Natarajan and Graph Dimensions}\label{sec:dimensions}
We recall two of the main generalizations of the VC
dimension to multiclass hypothesis classes.
\begin{definition}[Graph dimension]
Let $\ch\subseteq \left(\cy\cup\{\circleddash\}\right)^{\cx}$ be a hypothesis class. We say that $A\subseteq \cx$ is {\em $G$-shattered} if there exists $h:A\to\cy$ such that for every $B\subseteq A$ there is $h'\in\ch$ with $h(A)\subset\cy$ for which
\[
\forall x\in B,\;h'(x)=h(x)\text{ while }\forall x\in A\setminus B,\;h'(x)\ne h(x)~.
\]
The {\em graph dimension} of $\ch$, denoted $\Gdim(\ch)$, is the maximal cardinality of a $G$-shattered set.
\end{definition}
As the following theorem shows, the graph dimension essentially characterizes the ERM sample complexity.
\begin{theorem}[\cite{DanielySaBeSh11}]\label{thm:Gdim}
For every hypothesis class $\ch$ with graph dimension $d$,
\[
\Omega\left(\frac{d+\log(1/\delta)}{\epsilon}\right)
\le
m_{\erm}(\epsilon,\delta)
\le 
O\left(\frac{d\log(1/\epsilon)+\log(1/\delta)}{\epsilon}\right)~.
\]
\end{theorem}

\begin{definition}[Natarajan dimension]
Let $\ch\subseteq \left(\cy\cup\{\circleddash\}\right)^{\cx}$ be a hypothesis class. We say that $A\subseteq \cx$ is {\em $N$-shattered} if there exist $h_1,h_2:A\to\cy$ such that $\forall x\in A,\;h_1(x)\ne h_2(x)$ and for every $B\subseteq A$ there is $h\in\ch$ for which
\[
\forall x\in B,\;h(x)=h_1(x)\text{ while }\forall x\in A\setminus B,\;h(x)= h_2(x)~.
\]
The {\em Natarajan dimension} of $\ch$, denoted $\Ndim(\ch)$, is the maximal cardinality of an $N$-shattered set.
\end{definition}
\begin{theorem}[essentially \cite{Natarajan89b}]\label{thm:Ndim}
For every hypothesis class $\ch\subset\left(\cy\cup\{\circleddash\}\right)^{\cx}$ with Natarajan dimension $d$,
\[
\Omega\left(\frac{d+\log(1/\delta)}{\epsilon}\right)
\le
m_{\pac}(\epsilon,\delta)
\le 
O\left(\frac{d\log(|\cy|)\log(1/\epsilon)+\log(1/\delta)}{\epsilon}\right)~.
\]
\end{theorem}
We note that the upper bound in the last theorem follows from theorem
\ref{thm:Gdim} and the fact that (see \cite{Ben-DavidCeHaLo95}) for every hypothesis class $\ch$,
\begin{equation}\label{eq:NdimVsGdim}
\Gdim(\ch)\le 5 \log(|\cy|) \Ndim(\ch)~.
\end{equation}
We also note that \citep{DanielySaBeSh11} conjectured that the logarithmic factor of $|\cy|$ in \thmref{thm:Ndim} can be eliminated (maybe with the expense of poly-logarithmic factors of $\frac{1}{\epsilon},\frac{1}{\delta}$ and $\Ndim(\ch)$).

\subsection{Compression Schemes}\label{sec:compression}
A {\em compression scheme} of size $d$ for a class $\ch$ is a pair of functions:
\[
\Com:\cup_{m=d}^\infty (\cx \times \cy)^m\to (\cx \times \cy)^d
\text{ and }\DeCom: (\cx \times \cy)^d \to \cy^{\cx}~,
\]
with the property that for every realizable sample
\[
S=(x_1,y_1),\ldots,(x_m,y_m)
\]
it holds that, if $h=\DeCom\circ \Com (S)$ then
\[
\forall 1\le i\le m,\;\;y_i=h(x_i)~.
\]
Each compression scheme yields a learning algorithm, namely, $\DeCom\circ\Com$. It is known that the sample complexity of this algorithm is upper bounded by the size of the compression scheme. Precisely, we have:
\begin{theorem}[\cite{LittlestoneWa86}]\label{thm:compression}
Suppose that there exists a compression scheme of size $d$  for a class $\ch$. Then:
\begin{itemize}
\item The PAC sample complexity of $\ch$ is upper bounded by $O\left(\frac{d\log\left(1/\epsilon\right)+\frac{1}{\delta}}{\epsilon}\right)$
\item The agnostic PAC sample complexity of $\ch$ is upper bounded by $O\left(\frac{d\log\left(d/\epsilon\right)+\frac{1}{\delta}}{\epsilon^2}\right)$
\end{itemize}
\end{theorem}

\section{The Cantor classes}\label{sec:disc_classes}
\subsection{The first Cantor class}
Let $\cx$ be some finite set and let $\cy=2^{\cx}\cup\{*\}$. For every $A\subseteq \cx$ define $h_A:\cx\to\cy$ by
\[
h_A(x)=\begin{cases}
A & x\in A
\\
* & \text{otherwise}
\end{cases}~.
\]
Finally, let
\[
\ch_{\cx,\mathrm{Cantor}}=\left\{h_A\mid A\subset\cx\right\}~.
\]

\begin{lemma}[\cite{DanielySaBeSh11}]\label{lem:first_Cantor}
\
\begin{itemize}
\item The graph dimension of $\ch_{\cx,\mathrm{Cantor}}$ is $|\cx|$. Therefore, the ERM sample complexity of $\ch_{\cx,\mathrm{Cantor}}$ is $\Omega\left(\frac{|\cx|+\log(1/\delta)}{\epsilon}\right)$.
\item The Natarajan dimension of $\ch_{\cx,\mathrm{Cantor}}$ is $1$. Furthermore, the PAC sample complexity of $\ch_{\cx,\mathrm{Cantor}}$ is $O\left(\frac{\log(1/\delta)}{\epsilon}\right)$.
\end{itemize}
\end{lemma}
\begin{proof}
For the first part, it is not hard to see that the function
$f_\emptyset$ witnesses the $G$-shattering of $\cx$. The second part
follows directly from \lemref{lem:revealing_label}, given below.
\end{proof}

\begin{lemma}[essentially \cite{DanielySaBeSh11}]\label{lem:revealing_label}
Let $\ch\subset\cy^\cx$ be a hypothesis class with the following property: There is a label $*\in\cy$ such that, for every $f\in\ch$ and $x\in\cx$, either $f(x)=*$ or $f$ is the only function in $\ch$ whose value at $x$ is $f(x)$. Then, 
\begin{itemize}
\item The PAC sample complexity of $\ch$ is $\le\frac{\log(1/\delta)}{\epsilon}$.
\item $\Ndim(\ch)\le 1$.
\end{itemize} 
\end{lemma}
\begin{proof}
We first prove the second part. Assume on the way of contradiction that $\Ndim(\ch)> 1$. Let $\{x_1,x_2\}\subseteq \cx$ be an $N$-shattered set of cardinality $2$ and let $f_{1},f_{2}$ be two functions that witness the shattering. Since $f_{1}(x_1)\ne f_{2}(x_1)$, at least one of $f_{1}(x_1), f_{2}(x_1)$ is different from $*$. W.l.o.g, assume that $f_{1}(x_1)\ne *$. Now, by the definition of $N$-shattering, there is a function $f\in \ch_{\cy,\mathrm{Cantor}}$ such that $f(x_1)=f_1(x_1)$ and $f(x_2)=f_2(x_2)\ne f_1(x_2)$. However, the only function in $\ch$ satisfying $f(x_1)=f_1(x_1)$ is $f_{1}$. A contradiction.

We proceed to the first part. Assume w.l.o.g. the the function $f_*\equiv *$ is in $\ch$. Consider the following algorithm. Given a (realizable) sample
\[
(x_1,y_1),\ldots,(x_m,y_m),
\]
if $y_i=*$ for every $i$ then return the function $f_*$. Otherwise, return the hypothesis $h\in \ch$, that is consistent with the sample. Note the the existence of a consistent hypothesis is guaranteed, as the sample is realizable. This consistent hypothesis is also unique: if $y_i\ne *$ then, by the assumption on $\ch$, there is at most one function $f\in\ch$ for which $h(x_i)=y_i$.

This algorithm is an ERM with the following property: For every learnt hypothesis and underlying distribution, the algorithm might return only one out of two functions -- either $f_*$ or the learnt hypothesis. We claim that the sample complexity of such an ERM must be $\le \frac{\log(1/\delta)}{\epsilon}$. Indeed such an algorithm returns a hypothesis with error $\ge \epsilon$ only if:
\begin{itemize}
\item $\Err(f_*)\ge \epsilon$.
\item For every $i\in [m]$, $y_i=*$.
\end{itemize}
However, if $\Err(f_*)\ge \epsilon$, the probability that $y_i=*$ is $\le 1-\epsilon$. Therefore, the probability of the the second condition is $\le (1-\epsilon)^m\le e^{-m\epsilon}$, which is $\le \delta$ if $m\ge \frac{\log(1/\delta)}{\epsilon}$.
\end{proof}

\subsection{The second Cantor class}
Let $\tilde{\cy}$ be some non-empty finite set. Let $\cx=2^{\tilde{\cy}}$ and let $\cy=\tilde{\cy}
\cup\{*\}$. For every $y\in \tilde{\cy}$ define a function $h_y:\cx\to \cy$ by
\[
h_y(A)=\begin{cases}
y & y\in A
\\
* &\text{otherwise}
\end{cases}~.
\]
Also, let $h_{*}:\cx\to\cy$ be the constant function $*$. Finally, let $\ch_{\cy,\mathrm{Cantor}}=\{h_y\mid y\in\cy\}$.

\begin{lemma}\label{lem:second_Cantor}
\
\begin{itemize}
\item The graph dimension of $\ch_{\cy,\mathrm{Cantor}}$ is $\Theta\left(\log\left(|\cy|\right)\right)$. Therefore, the ERM sample complexity of $\ch_{\cy,\mathrm{Cantor}}$ is $\Omega\left(\frac{\log\left(|\cy|\right)+\log(1/\delta)}{\epsilon}\right)$.
\item The Natarajan dimension of $\ch_{\cy,\mathrm{Cantor}}$ is $1$. Furthermore, the PAC sample complexity of $\ch_{\cy,\mathrm{Cantor}}$ is $O\left(\frac{\log(1/\delta)}{\epsilon}\right)$.
\end{itemize}
\end{lemma}
\begin{proof} The second part of the lemma follows from 
\lemref{lem:revealing_label}. We proceed to the first part.  First, by
equation (\ref{eq:NdimVsGdim}) and the second part,
$\Gdim(\ch_{\cy,\mathrm{Cantor}})\le 5\log(|\cy |)$. It remains to
show that $\Gdim(\ch_{\cy,\mathrm{Cantor}})\ge \Omega\left(\log(|\cy
  |)\right)$. To do so, we must show that there are
$r=\Omega\left(\log(|\cy |)\right)$ sets
$\ca=\{A_1,\ldots,A_r\}\subseteq \cx$ such that $\ca$ is $G$-shattered
by $\ch_{\cy,\mathrm{Cantor}}$. To do so, we will use the
probabilistic method (see e.g. \cite{AlonSp00}). We will choose
$A_1,\ldots,A_r\subseteq \tilde{\cy}$ at random, such that each $A_i$ is
chosen uniformly at random from all subsets of $\tilde{\cy}$ (i.e.,
each $y\in \tilde{\cy}$ is independently chosen to be in $A_i$ with
probability $\frac{1}{2}$) and the different $A_i$'s are
independent. We will show that if $r=\lfloor
\frac{\log(|\cy|-1)}{2}\rfloor-2$, then with positive probability
$\ca=\{A_1,\ldots,A_r\}$ is $G$-shattered and $|\ca |=r$ (i.e., the
$A_i$'s are different).

Denote $d=|\tilde{\cy}|$. Let $\psi:[r]\to \cx$ be the (random)
function $\psi(i)=A_i$ and let $\phi:\cy\to \{0,1\}$ be the function
that maps each $y\in \tilde{\cy}$ to $1$ and $*$ to $0$. Consider the
(random) binary hypothesis class $\ch=\{\phi\circ h_y\circ\psi\mid
y\in \tilde{\cy}\}$. As we will show, for $r=\lfloor
\frac{\log(d)}{2}\rfloor-2$, $E[|\ch|]>2^r-1$. 
In particular, there exists some choice of $\ca=\{A_1,\ldots, A_r\}$ for which $|\ch |>2^r-1$. Fix those sets for a moment. Since always $|\ch |\le 2^r$, it must be the case that $|\ch | = 2^r$, i.e., $\ch=2^{[r]}$. By the definition of $\ch$, it follows that for every $B\subseteq [r]$, there is $h_y\in\ch_{\cy,\mathrm{Cantor}}$ such that for every $i\in B$, $h_y(A_i)=*$, while for every $i\notin B$, $h_y(A_i)\ne *$. It follows that $|\ca|=r$ and $\ca$ is $G$-shattered.

It remains to show that indeed, for $r=\lfloor \frac{\log(d)}{2}\rfloor-2$, $E[|\ch|]>2^r-1$. For every $S\subseteq [r]$, Let $\chi_S$ be the indicator random variable that is $1$ if and only if $1_S\in \ch$. We have
\begin{equation}\label{eq:4}
E[|\ch|]=E[\sum_{S\subseteq \mathbb [r]}\chi_S]=\sum_{S\subseteq \mathbb [r]}E[\chi_S]~.
\end{equation}
Fix some $ S\subseteq [r]$. For every $y\in \tilde{\cy}$ let $\chi_{S,y}$ be the indicator function that is $1$ if and only if $1_S=\phi\circ h_y\circ\psi$. Note that $\sum_{y\in\tilde{\cy}}\chi_{S,y}>0$ if and only if $\chi_S=1$. Therefore, $E[\chi_S]=\Pr\left(\chi_S=1\right)= \Pr\left(\sum_{y\in\tilde{\cy}}\chi_{S,y}>0\right)$. Observe that
\[
E[\sum_{y\in\tilde{\cy}}\chi_{S,y}]
=
\sum_{y\in\tilde{\cy}}\Pr\left(y\in A_i\text{ iff }i\in S\right)=
d\cdot 2^{-r}~.
\]
We would like to use Chebyshev's inequality for the sum $\sum_{y\in\tilde{\cy}}\chi_{S,y}$. For this to be effective, we show next that for different $y_1,y_2\in \tilde{\cy}$, $\chi_{S,y_1}$ and $\chi_{S,y_2}$ are uncorrelated. Note that $E[\chi_{S,y_1}\chi_{S,y_2}]$ is the probability that for every $i\in S$, $y_1,y_2\in A_i$ while for every $i\notin S$, $y_1,y_2\notin A_i$. It follows that
\[
E[\chi_{S,y_1}\chi_{S,y_2}]=2^{-2r}~.
\]
Therefore, $\mathrm{cov}(\chi_{S,y_1}\chi_{S,y_2})=E[\chi_{S,y_1}\chi_{S,y_2}]
-E[\chi_{S,y_1}]E[\chi_{S,y_2}]=2^{-2r}-2^{-r}2^{-r}=0$. We conclude that $\chi_{S,y_1}$ and $\chi_{S,y_2}$ are uncorrelated. Thus, by Chebyshev's inequality,
\begin{eqnarray*}
\Pr\left(\chi_{S}=0\right)&=& \Pr\left(\sum_{y\in\tilde{\cy}}\chi_{S,y}=0\right)
\\
&\le &
\Pr\left(\left|\sum_{y\in\tilde{\cy}}\chi_{S,y}-d\cdot 2^{-r}\right|\ge d\cdot 2^{-r-1}\right)
\\
&\le& \frac{2^{2r+2}}{d^2}\mathrm{var}\left(\sum_{y\in\tilde{\cy}}\chi_{S,y}\right)
\\
&=&\frac{2^{2r+2}}{d^2}\sum_{y\in\tilde{\cy}}\mathrm{var}\left(\chi_{S,y}\right)
\\
&\le &\frac{2^{2r+2}}{d^2}\sum_{y\in\tilde{\cy}}E[\chi_{S,y}]
\\
&=& \frac{2^{2r+2}}{d^2}d2^{-r}=\frac{2^{r+2}}{d}~.
\end{eqnarray*}
Remember that $r=\lfloor \frac{\log(d)}{2}\rfloor-2$, so that $d> 2^{2r+2}$. Hence, $E[\chi_S]=1-\Pr(\chi_S=0)\ge 1-2^{-r}$. Using equation (\ref{eq:4}), we conclude that
\[
E[|\ch |]>(1-2^{-r})2^r=2^r-1.
\]
\end{proof}

\section{Proofs}

\subsection{Some lemmas and additional notations}
Let $\cx',\cy'$ be another instance and label spaces. Let $\Gamma:\cx'\to\cx$ and $\Lambda:\cy\cup\{\circleddash\}\to\cy'\cup\{\circleddash\}$. We denote
\[
\Lambda\circ\ch\circ\Gamma=\{\Lambda\circ h\circ\Gamma\mid h\in\ch\}~.
\]
If $\Gamma$ (respectively $\Lambda$) is the identity function we simplify the above notation to $\Lambda\circ\ch$ (respectively $\ch\circ\Gamma$). We say that a hypothesis class $\ch'\subseteq \left(\cy'\cup\{\circleddash\}\right)^{\cx'}$ is {\em realizable} by $\ch\subseteq \left(\cy\cup\{\circleddash\}\right)^{\cx}$ if $\ch'
\subseteq \Lambda\circ\ch\circ\Gamma$ for some functions $\Gamma$ and $\Lambda$. Note that in this case, the different notions of sample complexity with respect to $\ch'$ are never larger than the corresponding notions with respect to $\ch$.

Let $\ch\subset\left(\cy\cup\{\circleddash\}\right)^\cx$ be a hypothesis class. The {\em disjoint
  union}  of $m$ copies of $\ch$ is the hypothesis class $\ch_m$ whose
instance space is $\cx_m:=\cx\times [m]$, whose label space is $\cy\cup\{\circleddash\}$,
and that is composed of all functions $f:\cx_m\to\cy\cup\{\circleddash\}$ whose
restriction to each copy of $\cx$ is a function in $\ch$ (namely, for every $i\in [m]$, the function $x\mapsto f(x,i)$ belongs to $\ch$). 

\begin{lemma}\label{lem:disjoint_embed}
Let $\ch\subseteq \cy^{\cx}$ be a hypothesis class. Let $\ch_m$ be a disjoint union of $m$ copies of $\ch$.
\begin{enumerate}
\item If $\ch$ is realized by $\ch_{\Psi}$ for some $\Psi:\cx'\times\cy'\to\mathbb R^d$, then $\ch_m$ is realized by $\ch_{\Psi_m}$ for some $\Psi_m:\cx'_m\times\cy'\to\mathbb R^{dm}$. Here, $\cx'_m$ is a disjoint union of $m$ copies of $\cx'$.
\item If $\ch$ is realized by $\ch_{\Psi,R}$ for some $\Psi:\cx'\times\cy'\to  B^d$, then $\ch_m$ is realized by $\ch_{\Psi_m,mR}$ for some $\Psi_m:\cx'_m\times\cy'\to B^{dm}$. Here, $\cx'_m$ is a disjoint union of $m$ copies of $\cx'$.
\item If $\ch$ is realized by $\ch_{d,k}$, then $\ch_m$ is realized by $\ch_{dm,k}$.
\item If $\ch$ is realized by $\ch_{d,k,R}$, then $\ch_m$ is realized by $\ch_{dm,k,mR}$.
\end{enumerate}
\end{lemma}
\begin{proof}
We prove only part 1. The remaining three are very similar. Let $\Gamma:\cx\to\cx',\Lambda:\cy'\to\cy$ be two mappings for which
\[
\ch\subseteq \Lambda\circ\ch_{\Psi}\circ\Gamma~.
\]
Let $\cx_m=\cx\times [m]$ be a disjoint union of $m$ copies of $\cx$. Let $T_i:\mathbb R^d\to\mathbb R^{dm}$ be the linear mapping that maps $e_j$ to $e_{(i-1)d+j}$. Define $\Psi_m:\cx_m\times\cy\to\mathbb R^{dm}$ by $\Psi_m((x,i),y)=T_i(\Psi(x,y))$. Define $\Gamma_m:\cx_m\to \cx'_m$ by $\Gamma_m(x,i)=(\Gamma(x),i)$. It is not hard to check that
\[
\ch_m\subseteq \Lambda\circ\ch_{\Psi_m}\circ\Gamma_m~.
\]
\end{proof}

\begin{lemma}\label{lem:disjoint_Gdim}
Let $\ch\subseteq \left(\cy\cup\{\circleddash\right)^{\cx}$ be a hypothesis class and let $\ch_m$ be a disjoint union of $m$ copies of $\ch$. Then $\Gdim(\ch_m)=m\cdot\Gdim(\ch)$.
\end{lemma}
\begin{proof}
A routine verification.
\end{proof}

\subsection{Proof of \thmref{thm:counter_example}}
For simplicity, we prove the theorem for $d$ even and $d=\infty$. For finite $d$, fix some $d$-elements set $\cx_d$. Let $\cy_d=2^{\cx_d}\cup\{*\}$. For $A\subseteq \cx_d$ define $h_A:\cx_d\to\cy_d$ by
\[
h_A(x)=
\begin{cases}
A & x\in A
\\
* & \text{otherwise}
\end{cases}~.
\]
Finally, let
\[
\ch_d=\left\{h_A\mid |A|= \frac{d}{2}\right\}~.
\]
We next define a ``limit" of the classes $\ch_d$. Suppose that the sets $\left\{\cx_d\right\}_{d\text{ is even integer}}$ are pairwise disjoint. Let $\cx_{\infty}=\cup_{d\text{ is even}}\cx_d$ and $\cy_\infty=\left(\cup_{d\text{ is even}}2^{\cx_d}\right)\cup\{*\}$. For $A\subseteq \cx_d$, extend $h_A:\cx_d\to \cy_d$ to a function  $h_A:\cx_\infty\to \cy_\infty$ by defining it to be $*$ outside of $\cx_d$. Finally, let
\[
\ch_\infty=\left\{h_A\mid \text{for some }d,\;A\subseteq \cx_d\text{ and }|A|= \frac{d}{2}\right\}~.
\]

We will use the following version of Chernoff's bound:
\begin{theorem}\label{thm:chern}
Let $X_1,\ldots,X_n\in \{0,1\}$ be independent random variables, $X=X_1+\ldots+X_n$ and $\mu=\E[X]$. Then $\Pr\left(X\ge 2\mu\right)\le \exp\left(-\frac{\mu}{3}\right)$.
\end{theorem}

We are now ready to prove \thmref{thm:counter_example}.  The first
part follows from \lemref{lem:revealing_label}. The last part is a
direct consequence of the first and second part.  We proceed to the
second part. For $d<\infty$, the task of properly learning $\ch_d$ can
be easily reduced to the task of properly learning
$\ch_{\infty}$. Therefore, the sample complexity of learning
$\ch_{\infty}$ by a proper learning algorithm is lower bounded by the
sample complexity of properly learning $\ch_d$. Therefore, it is
enough to prove the second part for finite $d$.

Fix some $x_0\in\cx$. Let $\epsilon>0$. Let $A\subset\cx_d\setminus\{x_0\}$ be a set with $\frac{d}{2}$ elements. Let $\cd_{A}$ be a distribution on $\cx_d\times \cy_d$ that assigns a probability of $1-16\epsilon$ to some point $(x_0,h_A(x_0))\in \cx_d\times\cy_d$ and is uniform on the remaining points of the form $\{(x,h_A(x))\mid x\not\in A\}$. 

We claim that there is some $A$ such that whenever $\ca$ runs on $\cd_{A}$ with $m\le \frac{1}{128}\frac{d}{\epsilon}$ examples, it outputs with probability $\ge \frac{1}{2}$ a hypothesis with error $\ge \epsilon$. This shows that for every $\delta<\frac{1}{2}$, $m_{\ca}(\epsilon,\delta)\ge \frac{1}{128}\frac{d}{\epsilon}$. Also, since $\ch_d$ contains two different function that agree on some point, by a standard argument, we have $m_{\ca}(\epsilon,\delta)=\Omega\left(\frac{\log(1/\delta)}{\epsilon}\right)$. Combining these two estimates, the proof is established.

It remains to show the existence of such $A$. Suppose that $A$ is chosen uniformly at random among all subsets of $\cx_d\setminus\{x_0\}$ of size $\frac{d}{2}$.
Let $X$ be the random variable counting the number of samples, out of
$\frac{1}{128}\frac{d}{\epsilon}$ i.i.d. examples drawn from $\cd_A$,
which are not $(x_0,h_A(x_0))$. We have
$\E[X]=\frac{1}{8}d$. Therefore, by Chernoff's bound \ref{thm:chern},
with probability $>1-\exp\left(-\frac{d}{24}\right)>\frac{1}{2}$, the
algorithm will see less than $\frac{d}{4}$ examples whose instance is
from $\cx\setminus\{x_0\}\setminus A$. Conditioning on this event, $A$
is a uniformly chosen random set of size $\frac{d}{2}$ that is chosen
uniformly from all subsets of a set $\cx'\subset\cx$ with $|\cx'|\ge\frac{3}{4}d$ ($\cx'$ is the set of all points that are not present in the sample), and
the hypothesis returned by the algorithm is $h_{B}$, where $B\subset\cx$ is a subset of size $\frac{d}{2}$ that is independent from $A$. It is not hard to see that in this case $\mathbb{E}|B\setminus A|\ge \frac{1}{6}d$. Hence, there exists some $A$ for which, with probability $>\frac{1}{2}$ over the choice of the sample, $|B\setminus A|\ge \frac{1}{6}d$. For such $A$ we have, since $h_B$ errs on all elements in $B\setminus A$ and the probability of each such element is $\ge\frac{16\epsilon}{\frac{d}{2}}=\frac{32}{d}\epsilon$,
\[
\Err_{\cd_A}(h_B)\ge |B\setminus A|\frac{32\epsilon}{d}
\ge
\frac{d}{6}\frac{32\epsilon}{d}>\epsilon
\]
with probability $>\frac{1}{2}$ over the choice of the sample.

\subsection{Proof of \lemref{lem:out_deg}}
We first prove it to finite hypergraphs. We use induction on the number of vertices. By assumption, $d(G)\le d$. Therefore, there is $v_0\in V$ with $d(v_0)\le d$. Let $G'=(V',E')=G[V\setminus\{v_0\}]$. By the induction hypothesis, there exists an orientation $h':E'\to V'$ with maximal out-degree $d$. We define an orientation $h:E\to V$ by
\[
h(e)=\begin{cases}
v & e=\{v_0,v\}
\\
h'(e\setminus \{v_0\}) & \text{otherwise}
\end{cases}
\]
The lemma extend to the case where $\cy$ is infinite by a standard application of the compactness theorem for propositional calculus.

\subsection{Proof of theorem \ref{thm:main_opt_alg_pac}}
Let $\ca$ be some learning algorithm, and denote by $\ci$ the one inclusion algorithm. Suppose that we run $\ca$ on $m_{\ca,\ch}\left(\frac{\epsilon}{2},\frac{\epsilon}{2}\right)$ examples, obtain a 
hypothesis $h$ and predict $h(x)$ on some new example. The probability of error if $\le \left(1-\frac{\epsilon}{2}\right)\frac{\epsilon}{2}+\frac{\epsilon}{2}\le \epsilon$. By theorem \ref{thm:main_opt_alg}, it follows that
\[
m_{\ca,\ch}\left(\frac{\epsilon}{2},\frac{\epsilon}{2}\right)
\ge
\min\left\{m\mid \frac{1}{2e}\frac{\mu_{\ch}(m)}{m}\le\epsilon\right\}=:\bar{m}~.
\]
Now, if we run the one inclusion algorithm on $\bar{m}$ examples then, again by theorem \ref{thm:main_opt_alg}, the probability that the hypothesis it return will err a new example is $\le 2e\epsilon$. Therefore, the probability that the error of the returned hypothesis is $\ge 4e\epsilon$ is $\le \frac{1}{2}$. In follows that
\[
\bar{m}\ge m_{\ci,\ch}\left(4e\epsilon,\frac{1}{2}\right)~.
\]
Combining the two inequalities, we obtain that
\[
m_{\ci,\ch}\left(4e\epsilon,\frac{1}{2}\right)\le m_{\ca,\ch}\left(\frac{\epsilon}{2},\frac{\epsilon}{2}\right)
\]
Since this is true for every algorithm $\ca$, we have
\[
m_{\ci,\ch}\left(4e\epsilon,\frac{1}{2}\right)\le m_{\mathrm{PAC},\ch}\left(\frac{\epsilon}{2},\frac{\epsilon}{2}\right)\le  m_{\mathrm{PAC},\ch}\left(\frac{\epsilon}{4},\frac{1}{2}\right)\cdot O\left(\log(1/\epsilon)\right)
\]
Here, the last inequality follows by a standard repetition argument. Equivalently,
\[
m_{\ci,\ch}\left(\epsilon,\frac{1}{2}\right)\le m_{\mathrm{PAC},\ch}\left(\frac{\epsilon}{16e},\frac{1}{2}\right)\cdot O\left(\log(1/\epsilon)\right)
\]
Again, using a repetition argument we conclude that
\[
m_{\overline{\ci},\ch}(\epsilon,\delta)\le
m_{\ci,\ch}\left(\frac{\epsilon}{2},\frac{1}{2}\right)\cdot O\left(\log(1/\delta)\right)
\le
m_{\mathrm{PAC},\ch}\left(\frac{\epsilon}{32e},\frac{1}{2}\right)\cdot O\left(\log(1/\delta)\log(1/\epsilon)\right)
\]

\subsection{Validity of the compression scheme given in
  \secref{subsec:compressionDim}} \label{subsec:compressionDimValidity}

It is not hard to see that the hypothesis we output is the minimal-norm vector $w\in \conv(Z)$ (where $Z$ is the set defined in the compression step).
It is left to show that $w$ makes no errors on the original sample. Indeed, otherwise there exists $z\in Z$ for which $\inner{w,z}\le 0$. By claim \ref{claim:comp_1}, $z\ne 0$.
For $\alpha=\frac{\|w\|^2}{\|z\|^2+\|w\|^2}\in (0,1)$, let $w'=(1-\alpha)w+\alpha z$. We have that $w'\in \conv(Z)$. Moreover,
\begin{align*}
\|w'\|^2 &= (1-\alpha)^2\|w\|^2+\alpha^2\|z\|^2+2\alpha(1-\alpha)\inner{w,z}
\le  (1-\alpha)^2\|w\|^2+\alpha^2\|z\|^2
\\
&=\frac{\|z\|^4
  \|w\|^2+\|w\|^4\|z\|^2}{\left(\|z\|^2+\|w\|^2\right)^2} 
= \frac{\|z\|^2 \|w\|^2}{\|z\|^2+\|w\|^2}
< \|w\|^2~.
\end{align*}
This contradicts the minimality of $w$.
It only remains to prove the following claim, which was used in the analysis.
\begin{claim}\label{claim:comp_1}
Let $(x_1,y_1),\ldots,(x_m,y_m)$ be a realizable sample and let $Z$ be the set of all vectors of the form $\Psi(x_i,y_i)-\Psi(x_i,y)$ for $y\ne y_i$. Then $0\not\in \conv(Z)$.
\end{claim}
\begin{proof}
  Since the sample is realizable, there exists a vector $w$ in
  $\mathbb R^d$ for which, $\forall z\in Z,\;\inner{w,z}>0$.  Clearly,
  this holds also for every $z\in\conv(Z)$, hence $0\not\in
  \conv(Z)$.
\end{proof}

\subsection{Proof of the second part of \thmref{thm:gen_non_margin}}

Without loss of generality, we assume that $\cy$ consists of $2^n+1$
elements for some natural number $n$ (otherwise, use only $2^n+1$
labels, where $n$ is the largest number satisfying $2^n+1\le
|\cy|$). Let $\cx$ be a set consisting of $n$ elements. By renaming
the names of the labels, we can assume that $\cy=2^{\cx}\cup\{*\}$. By
\lemref{lem:second_Cantor}, the ERM sample complexity of
$\ch_{\cx,\mathrm{Cantor}}$ is
$\Omega\left(\frac{\log(|\cy|)+\log(1/\delta)}{\epsilon}\right)$. We
will show that there exists a function $\Psi:\cx\times\cy\to \mathbb
R^3$, such that $\ch_{\cx,\mathrm{Cantor}}$ is realized by
$\ch_{\Psi}$. It follows that the ERM sample complexity of
$\ch_{\Psi}$ is also
$\Omega\left(\frac{\log(|\cy|)+\log(1/\delta)}{\epsilon}\right)$. Therefore,
the second part of \thmref{thm:gen_non_margin} is proved for
$d=3$. The extension of the result to general $d$ follows from
\lemref{lem:disjoint_embed}.

{\bf Definition of $\Psi$:} Denote $k=2^{|\cx|}$ and let $f:2^{\cx}\to \left\{0,1,\ldots, k-1\right\}$ be some one-to-one mapping.  For $A\subseteq \cx$ define
\[
\phi(A)=\left(\cos\left(\frac{2\pi f(A)}{k}\right),\sin\left(\frac{2\pi f(A)}{k}\right),0\right)~.
\]
Also, define
\[
\phi(*)=\left(0,0,1\right)~.
\]
Note that for different subsets $A,B\subseteq \cx$ we have that
\begin{equation}\label{eq:2}
\inner{\phi(A),\phi(B)}=
\cos\left(\frac{2\pi(f(A)-f(B))}{k}\right)
\le \cos\left(\frac{2\pi}{k}\right)<\frac{1}{2}+\frac{1}{2}\cos\left(\frac{2\pi}{k}\right)<1
\end{equation}
Define $\Psi:\cx\times\cy\to \mathbb R^3$ by
\[
\forall A\subset\cx,\;\;\Psi(x,A)=\begin{cases}
\phi(A) & x\in A
\\
0 & x\not\in A
\end{cases}
\]
\[
\Psi(x,*)=\left(\frac{1}{2}+\frac{1}{2}\cos\left(\frac{2\pi}{k}\right)\right) \cdot \phi(*)
\]
\begin{claim}
$\ch_{\cx,\mathrm{Cantor}}$ is realized by $\ch_{\Psi}$.
\end{claim}
\proof
We will show that $\ch_{\cx,\mathrm{Cantor}}\subseteq \ch_{\Psi}$. Let $B\subseteq \cx$. We must show that $h_{B}\in \ch_{\Psi}$. Let $w\in\reals^3$ be the vector
\[
w=\phi(B)+\phi(*)~.
\]
We claim that for the function $h_w\in\ch_{\Psi}$, defined by $w$ we have $h_w=h_B$. Indeed, let $x\in \cx$ we split into the cases $x\in B$ and $x\notin B$. 

{\bf Case 1 ($x\in B$):}
We must show that $h_w(x)=B$. That is, for every $y\in \cy\setminus\{B\}$,
\[
\inner{w,\Psi(x,B)}>\inner{w,\Psi(x,y)}~.
\]
Note that
\[
\inner{w,\Psi(x,B)}=\left\langle\phi(B)+\phi(*),\phi(B)\right\rangle=1~.
\]
Therefore, for every $y\in \cy\setminus\{B\}$, we must show that $1>\inner{w,\Psi(x,y)}$. We split into three cases. If $y=A$ for some $A\subseteq \cx$ and $x\in A$ then, using equation (\ref{eq:2}),
\[
\inner{w,\Psi(x,y)}=\left\langle\phi(B)+\phi(*),\phi(A)\right\rangle=\left\langle\phi(B),\phi(A)\right\rangle<1~.
\]
If $y=A$ for some $A\subseteq \cx$ and $x\not\in A$ then,
\[
\inner{w,\Psi(x,y)}=\left\langle\phi(B)+\phi(*),0\right\rangle=0<1~.
\]
If $y=*$ then,
\[
\inner{w,\Psi(x,y)}=\left\langle\phi(B)+\phi(*),\left(\frac{1}{2}+\frac{1}{2}\cos\left(\frac{2\pi}{k}\right)\right)\cdot\phi(*)\right\rangle=
\frac{1}{2}+\frac{1}{2}\cos\left(\frac{2\pi}{k}\right)<1~.
\]

{\bf Case 2 ($x\notin B$):} We must show that $h_w(x)=*$. That is, for every $A\in \cy\setminus\{*\}$,
\[
\inner{w,\Psi(x,*)}>\inner{w,\Psi(x,A)}~.
\]
Note that
\[
\inner{w,\Psi(x,*)}=\left\langle\phi(B)+\phi(*),\left(\frac{1}{2}+\frac{1}{2}\cos\left(\frac{2\pi}{k}\right)\right)\phi(*)\right\rangle=\frac{1}{2}+\frac{1}{2}\cos\left(\frac{2\pi}{k}\right)~.
\]
Therefore, for every $A\in \cy\setminus\{B\}$, we must show that $\frac{1}{2}+\frac{1}{2}\cos\left(\frac{2\pi}{k}\right)>\inner{w,\Psi(x,A)}$. Indeed, if $x\in A$ then $A\ne B$ (since $x\notin B$). Therefore, using equation (\ref{eq:2}),
\[
\inner{w,\Psi(x,A)}=\left\langle\phi(B)+\phi(*),\phi(A)\right\rangle=\left\langle\phi(B),\phi(A)\right\rangle<\frac{1}{2}+\frac{1}{2}\cos\left(\frac{2\pi}{k}\right)~.
\]
If $x\notin A$ then 
\[
\inner{w,\Psi(x,A)}=\left\langle\phi(B)+\phi(*),0\right\rangle=
0<\frac{1}{2}+\frac{1}{2}\cos\left(\frac{2\pi}{k}\right)~.
\]

\subsection{Proof of \thmref{thm:gen_margin}}
To prove the first part of \thmref{thm:gen_margin}, we will rely again on \thmref{thm:compression}. We will show a compression scheme of size $O(R)$, which is based on the multiclass perceptron. This compression scheme is not new. However, for completeness, we briefly survey it next. Recall that the multiclass perceptron is an online classification algorithm. At each step it receives an instance and tries to predict its label, based on the observed past. The two crucial properties of the preceptron that we will rely on are the following:
\begin{itemize}
\item If the perceptron runs on a sequence of examples that is realizable by $\ch_{\Psi,R}$, then it makes at most $O(R)$ mistakes.
\item The predictions made by the perceptron algorithm, are affected only by previous {\em erroneous} predictions.
\end{itemize}
Based on these two properties,  the compression scheme proceeds as follows: Given a realizable sample $S=\{(x_1,y_1),\ldots,(x_m,y_m)\}$, it runs the preceptron algorithm $\Omega(R)$ times on the sequence $(x_1,y_1),\ldots,(x_m,y_m)$ (without a reset between consecutive runs). By the first property, in at least one of these runs, the preceprton will make no mistakes on the sequence $(x_1,y_1),\ldots,(x_m,y_m)$ (otherwise, there would be $\Omega(R)$ mistakes in total). The output of the compression step would be the erroneous examples previous to this sequence. By the first property, the number of such examples is $O(R)$. 
The decompression will run the preceptron on these examples, and output the hypothesis $h:\cx\to\cy$, such that $h(x)$ is the prediction of the perceptron on $x$, after operating on these examples. By the second property, $h$ is correct on every $x_i$.

We proceed to the second part. By \lemref{lem:disjoint_Gdim} and
\lemref{lem:first_Cantor}, it is enough to show that a disjoint
union of $\Omega(R)$ copies of $\ch_{\cx,\mathrm{Cantor}}$, with
$|\cx|=\Omega(\log(|\cy|))$, can be realized by $\ch_{\Psi,R}$ for an
appropriate mapping $\Psi:\cx\times \cy\to B^d$ for some $d>0$. By
\lemref{lem:disjoint_embed}, it is enough to show that, for some
universal constant $C>0$, $\ch_{\cx,\mathrm{Cantor}}$, with
$|\cx|=\Omega(\log(|\cy|))$, can be realized by $\ch_{\Psi,C}$ for an
appropriate mapping $\Psi:\cx\times \cy\to B^d$ for some $d>0$.

Without loss of generality, we assume that $|\cy|-1$ is a power of $2$ (otherwise, use only $k$ labels, where $k$ is the largest integer such that $k-1$ is a poser of $2$ and $k\le |\cy|$). Denote $k=|\cy|-1$. Fix some finite set $\cx$ of cardinality $\log(|\cy|-1)$. By renaming the labels, we can assume that $\cy=2^{\cx}\cup\{*\}$.

Let $\{e_{y}\}_{y\in \cy}$ be a collection of unit vectors in $\reals^d$ with the property that for $y_1\ne y_2$,
\begin{equation}\label{eq:3}
|\inner{e_{y_1}, e_{y_2}}|<\frac{1}{100}~.
\end{equation}
\begin{remark}
Clearly, it is possible to find such a collection when $d=k+1$ (simply take $\{e_{y}\}_{y\in \cy}$ to be an orthogonal basis of $\reals^{k+1}$). However, equation (\ref{eq:3}) requires the collection to be just ``almost orthogonal". Such a collection can be found in $\reals^d$ for $d=O(\log(k))$ (see, e.g. \cite{Matousek02}, chapter 13).
\end{remark}
Define $\Psi:\cx\times\cy\to  B^d$ by
\[
\forall A\subset\cx,\;\;\Psi(x,A)=\begin{cases}
e_A & x\in A
\\
0 & x\not\in A
\end{cases}
\]
\[
\Psi(x,*)=e_*
\]
The following claim establishes the proof of \thmref{thm:gen_margin}.
\begin{claim}
$\ch_{\cx,\mathrm{Cantor}}$ is realized by $\ch_{\Psi,8}$.
\end{claim}
\proof
We will show that
$\ch_{\cx,\mathrm{Cantor}}\subseteq \ch_{\Psi,8}$. Let $B\subseteq \cx$. We must show that $h_B\in\ch_{\Psi,8}$. Let $w=W\cdot(e_{B}+\frac{1}{2}e_{*})$ for $W=\frac{100}{45}$. We claim that the hypothesis in $\ch_{\Psi,8}$ that corresponds to $w$ is $h_B$. Indeed, let $x\in \cx$. We split into the cases $x\in B$ and $x\notin B$. 

{\bf Case 1 ($x\in B$):} We must show that for every $y\in \cy\setminus\{B\}$,
\[
\inner{w,\Psi(x,B)}\ge 1+\inner{w,\Psi(x,y)}~.
\]
Note that
\[
\inner{w,\Psi(x,B)}=\left\langle W\cdot\left(e_{B}+\frac{1}{2}e_{*}\right), e_B \right\rangle=W\left(1+\frac{1}{2}\inner{e_*,e_B}\right)\ge W\left(1-\frac{1}{100}\right)~.
\]
Now, if $y\in \cy\setminus\{B\}$ then either $y\subseteq \cx$ and $x\not\in y$. In this case, $\inner{w,\Psi(x,y)}=\inner{w,0}=0$. In the remaining cases,
\[
\inner{w,\Psi(x,y)}=\left\langle W\cdot\left(e_{B}+\frac{1}{2}e_{*}\right), e_y \right\rangle=W\left(\inner{e_y,e_B}+\frac{1}{2}\inner{e_*,e_B}\right)\le W\frac{1}{50}~.
\]
It follows that
\[
\inner{w,\Psi(x,B)}-\inner{w,\Psi(x,y)}\ge \frac{24}{25}W\ge 1~.
\]

{\bf Case 2 ($x\notin B$):} We must show that for every $y\in \cy\setminus\{*\}$,
\[
\inner{w,\Psi(x,*)}\ge 1+\inner{w,\Psi(x,y)}~.
\]
Note that
\[
\inner{w,\Psi(x,*)}=\left\langle W\cdot\left(e_{B}+\frac{1}{2}e_{*}\right), e_* \right\rangle=W\left(\inner{e_B,e_*}+\frac{1}{2}\right)\ge W\left(\frac{1}{2}-\frac{1}{100}\right)~.
\]
Now, suppose that $A=y\in \cy\setminus\{*\}$. If $x\notin A$ then,
\[
\inner{w,\Psi(x,y)}=\left\langle W\cdot\left(e_{B}+\frac{1}{2}e_{*}\right), 0 \right\rangle=0\le \frac{1}{25}W~.
\]
If $x\in A$ then $A\ne B$. Therefore,
\[
\inner{w,\Psi(x,y)}=\left\langle W\cdot\left(e_{B}+\frac{1}{2}e_{*}\right), e_A \right\rangle=W\left(\inner{e_B,e_A}+\frac{1}{2}\inner{e_*,e_A}\right)\le \frac{1}{25}W~.
\]
It follows that
\[
\inner{w,\Psi(x,*)}-\inner{w,\Psi(x,y)}\ge \frac{45}{100}W\ge 1~.
\]

\subsection{Proof of \thmref{thm:struc_margin}}
The first part of the theorem follows directly from the first part of \thmref{thm:gen_margin}. We proceed to the second part. First, we note that $\ch_{d,t,2,R}$ can be realized by $\ch_{d,t,q,R}$. Therefore, it is enough to restrict to the case $q=2$. To simplify notations, we denote $\ch_{d,t,2,R}$ by $\ch_{d,t,R}$. Also, the label space of $\ch_{d,t,R}$ will be $\{0,1\}^t$ instead of $[2]^q$.

By \lemref{lem:disjoint_Gdim} and \lemref{lem:second_Cantor}, it is
enough to show that a disjoint union of $\Omega(R)$ copies of
$\ch_{\cy,\mathrm{Cantor}}$, with $|\cy|=\Omega(t)$, can be realized
by $\ch_{d,t,R}$ for $d\ge (t+1)R$. By \lemref{lem:disjoint_embed},
it is enough to show that, for some universal constant $C>0$,
$\ch_{\cy,\mathrm{Cantor}}$, with $|\cy|=t+1$, can be realized by
$\ch_{t+1,t,C}$. Indeed:

\begin{claim}
Let $\tilde{\cy}=[t]$. The class $\ch_{\cy,\mathrm{Cantor}}$ is realized by $\ch_{t+1,t,128}$.
\end{claim}
\begin{proof}
Recall that the instance space of $\ch_{\cy,\mathrm{Cantor}}$ is $\cx=2^{[t]}$. Also, let $e_*:=e_{t+1}\in B^{t+1}$.
Consider the mapping $\Gamma:\cx\to (B^{t+1})^{t}$ defined as follows. For every $A\in\cx$, $\Gamma(A)$ is the matrix whose $i$'th column is $\frac{1}{2}e_{i}+\frac{1}{4}e_{*}$ if $i\in A$ and $\frac{1}{4}e_{*}$ otherwise. Let $\Lambda:\{0,1\}^t\cup\{\circleddash\}\to [t]\cup\{*\}$ be any mapping that maps $e_i\in \{0,1\}^t$ to $i$ and $0\in \{0,1\}^t$ to $*$. To establish the claim we will show that
\[
\ch_{\cy,\mathrm{Cantor}}\subseteq \Lambda\circ\ch_{t+1,t,128}\circ\Gamma~.
\]
We must show that for every $i\in [t]$, $h_i\in \Lambda\circ\ch_{t+1,t,128}\circ\Gamma$ and that $h_*\in \Lambda\circ\ch_{t+1,t,128}\circ\Gamma$. We start with $h_i$. Let $W\in M_{(t+1)\times 2}$ be the matrix whose left column is $0$ and whose right column is $8e_i-8e_{*}$. Let $h_{W}\in \ch_{t+1,t,128}$ be the hypothesis corresponding to $W$. We claim that $h_{i}=\Lambda\circ h_W\circ\Gamma$. Indeed, let $A\in\cx$. We must show that $\Lambda(h_W(\Gamma(A)))=h_i(A)$. By the definition of $\Lambda$ and $h_i$, it is enough to show that $h_W(\Gamma(A))=e_i$ if $i\in A$, while $h_W(\Gamma(A))=0$ otherwise. 
Let $\Psi:(B^{t+1})^t\times\{0,1\}^t\to M_{t+1,2}$ be the mapping for which $\ch_{t+1,t,128}=\ch_{\Psi,128}$.
Since the left column of $W$ is zero, we have that $\inner{W,\Psi(\Gamma(A),0)}=0$, and for $0\ne y\in\{0,1\}^t$,
\begin{eqnarray*}
\inner{W,\Psi(\Gamma(A),y)}&=&\frac{1}{2\cdot |\{j\mid y_j=1\}|}\sum_{j\mid y_j=1}\inner{4e_i-4e_{*},(\Gamma(A))^j}
\\
&=&\frac{1}{|\{j\mid y_j=1\}|}\sum_{j\mid y_j=1}(2\cdot 1[i=j\text{ and }i\in A]-1)
\\
&=&\frac{2\cdot 1[i\in A\text{ and }y_i=1]}{|\{j\mid y_j=1\}|}-1~.
\end{eqnarray*}
It follows that if $i\in A$ then $\inner{W,\Psi(\Gamma(A),e_i)}=1$ while $\inner{W,\Psi(\Gamma(A),y)}\le 0$ for every $y\ne e_i$. Therefore, $h_W(\Gamma(A))=e_i$.
If $i\notin A$ then $\inner{W,\Psi(\Gamma(A),0)}=0$ while $\inner{W,\Psi(\Gamma(A),y)}\le -1$ for every $y\ne 0$. Therefore $h_W(\Gamma(A))=0$.

The fact that $h_*\in \Lambda\circ\ch_{t+1,t,128}\circ\Gamma$ follows from a similar argument, where $W\in M_{(t+1)\times 2}$ is the matrix whose left column is $0$ and whose right column is $-8e_{*}$. It is not hard to see that if $h_{W}\in \ch_{t+1,t,128}$ is the hypothesis corresponding to $W$, we have $h_{*}=\Lambda\circ h_W\circ\Gamma$.
\end{proof}

\subsection{Proof of \thmref{thm:struc_non_margin}}
The first part of the theorem follows directly from the first part of
\thmref{thm:gen_non_margin}. We proceed to the second
part. First, by the following lemma, it is enough to restrict
ourselves to the case $q=2$. Given two hypothesis classes $\ch\subseteq \cy^{\cx}$ and $\ch'\subseteq {\cy'}^{\cx'}$, we say that $\ch'$ {\em finitely realizes} $\ch$ if, for every finite $\cu\subset\cx$, $\ch'$ realizes $\ch|_{\cu}$. It is clear that in this case $\Gdim\left(\ch'\right)\ge \Gdim\left(\ch\right)$.
\begin{lemma}
For every $d,t$ and $q\ge 2$, a disjoint union of $\lfloor\frac{q}{2}\rfloor$ copies of $\ch_{d,t,2}$ is finitely realized by $\ch_{d+2,t,q}$
\end{lemma}
\begin{proof}
For simplicity, assume that $q$ is even and let $r=\frac{q}{2}$. Let $X_1,\ldots,X_r$ be finite subsets of $M_{d,t}$. We should show that there is a mapping $\Gamma:X_1\dot\cup\ldots\dot\cup X_r\to M_{d+2,t}$ and a mapping $\Lambda:[q]^t\to [2]^t$ such that
\begin{equation}\label{eq:5}
\left(\ch_{d,t,2}\right)_m|_{X_1\dot\cup\ldots\dot\cup X_r}
\subset
\left(\Lambda\circ \ch_{d+2,t,q}\circ \Gamma\right)|_{X_1\dot\cup\ldots\dot\cup X_r}
\end{equation}
For $x\in X_j$ we define
\[
\Gamma(x)=\left(x^T,\cos\left(j\frac{ 2\pi}{r}\right),\sin\left(j\frac{ 2\pi}{r}\right)\right)^T
\]
Also, let $\lambda:[q]\to [2]$ be the function that maps odd numbers to $1$ and even numbers to $2$. Finally, define $\Lambda: [q]^t\to [2]^t$ by $\Lambda(y_1,\ldots,y_t)=(\lambda(y_1),\ldots,\lambda(y_t))$. We claim that (\ref{eq:5}) holds with these $\Lambda$ and $\Gamma$.

Indeed, let $W_1,\ldots, W_r\in M_{d\times 2}$. We should show that the function $g\in \left(\ch_{d,t,2}\right)_m|_{X_1\dot\cup\ldots\dot\cup X_r}$ defined by these function is of the form $\left(\Lambda\circ h\circ \Gamma\right)|_{X_1\dot\cup\ldots\dot\cup X_r}$ for some $h\in \ch_{d+2,t,q}$. Fix $M>0$ and let $h$ be the hypothesis defined by the matrix $W\in M_{d+2,q}$ defined as follows
\[
W=\begin{bmatrix}
W_1^1 & W_1^2 & W^1_2 & W^2_2 & \cdots & W^1_r & W_r^2
\\
M\cos\left(\frac{ 2\pi}{r}\right) & M\cos\left(\frac{ 2\pi}{r}\right)& 
M\cos\left(2\frac{ 2\pi}{r}\right) & M\cos\left(2\frac{ 2\pi}{r}\right)& \cdots & 
M\cos\left(r\frac{ 2\pi}{r}\right) & M\cos\left(r\frac{ 2\pi}{r}\right)
\\
M\sin\left(\frac{ 2\pi}{r}\right) & M\sin\left(\frac{ 2\pi}{r}\right)& 
M\sin\left(2\frac{ 2\pi}{r}\right) & M\sin\left(2\frac{ 2\pi}{r}\right)& \cdots &
M\sin\left(r\frac{ 2\pi}{r}\right) & M\sin\left(r\frac{ 2\pi}{r}\right)
\end{bmatrix}
\]
It is not hard to check that for large enough $M$, $g=\left(\Lambda\circ h\circ \Gamma\right)|_{X_1\dot\cup\ldots\dot\cup X_r}$
\end{proof}

Next we prove \thmref{thm:struc_non_margin} for $q=2$. To
simplify notation, we let $\ch_{d,t}:=\ch_{d,t,2}$.  We make one
further reduction, showing that it is enough to prove the theorem for
the case $d=3$.  Indeed, by \lemref{lem:disjoint_Gdim} and
\lemref{lem:second_Cantor}, it is enough to show that a disjoint union
of $\Omega(d)$ copies of $\ch_{\cy,\mathrm{Cantor}}$, with
$|\cy|=\Omega(k)$, can be realized by $\ch_{d,t}$. By
\lemref{lem:disjoint_embed}, it is enough to show that, for some
universal constant $C>0$ (we will take $C=3$),
$\ch_{\cy,\mathrm{Cantor}}$, with $|\cy|=t+1$, can be realized by
$\ch_{C,t}$. Indeed:
\begin{claim}
Let $\tilde{\cy}=[t]$. The class $\ch_{\cy,\mathrm{Cantor}}$ is realized by $\ch_{3,t}$.
\end{claim}
\begin{proof}[(sketch)]
The proof is similar to the proof of the second part of \thmref{thm:struc_margin}. Recall that the instance space of $\ch_{\cy,\mathrm{Cantor}}$ is $\cx=2^{[t]}$. For $i\in [t]$ define $\phi(i)=\left(\cos\left(\frac{i2\pi}{t}\right),\sin\left(\frac{i2\pi}{t}\right),0\right)$. Also, let
\[
\phi(*)=\left(0,0,\frac{1}{2}+\frac{1}{2}\cos\left(\frac{2\pi}{t}\right)\right)~.
\]
Consider the mapping $\Gamma:\cx\to (B^{3})^{t}$ defined as follows. For every $A\in\cx$, $\Gamma(A)$ is the matrix whose $i$'th column is $\frac{1}{2}\phi(i)+\frac{1}{2}\phi(*)$ if $i\in A$ and $\frac{1}{2}\phi(*)$ otherwise. Let $\Lambda:\{0,1\}^t\cup\{\circleddash\}\to [k]\cup\{*\}$ be any mapping that maps $e_i\in \{0,1\}^t$ to $i$ and $0\in \{0,1\}^t$ to $*$. To establish the claim we will show that
\[
\ch_{\cy,\mathrm{Cantor}}\subseteq \Lambda\circ\ch_{3,t}\circ\Gamma~.
\]
We must show that for every $i\in [t]$, $h_i\in \Lambda\circ\ch_{3,t}\circ\Gamma$ and that $h_*\in \Lambda\circ\ch_{3,t}\circ\Gamma$. We start with $h_i$. Let $W\in M_{3\times 2}$ be the matrix whose left column is $0$ and whose right column is $\phi(i)-e_3$. It is not hard to see that if $h_{W}\in \ch_{3,t}$ is the hypothesis corresponding to $W$, we have $h_{i}=\Lambda\circ h_W\circ\Gamma$.

For $h_*$, let $W\in M_{3\times 2}$ be the matrix whose left column is $0$ and whose right column is $-e_{3}$. It is not hard to see that for $h_{W}\in \ch_{3,t}$, we have $h_{*}=\Lambda\circ h_W\circ\Gamma$.
\end{proof}

\subsection{Proof of \thmref{thm:main_nat_dim}}
\begin{theorem}\label{thm:main_nat_dim}
For every $\Psi:\cx\times\cy\to\mathbb R^d$, $\Ndim (\ch_\Psi) \le d$.
\end{theorem}
\begin{proof}
Let $C\subseteq \cx$ be an $N$-shattered set, and let $f_0,f_1:C\to \cy$ be two functions that witness the shattering. We must show that $|C|\le d$. For every $x\in C$ let $\rho(x)=\Psi(x,f_0(x))-\Psi(x,f_1(x))$. We claim that $\rho(C)=\{\rho(x)\mid x\in C\}$ consists of $|C|$ elements (i.e. $\rho$ is one to one) and is shattered by the binary hypothesis class of homogeneous linear separators on $\mathbb{R}^d$,
\[
\ch=\{x\mapsto \sign(\langle w,x\rangle)\mid w\in \mathbb R^d\}~.
\]
Since $\VCdim(\ch)=d$, it will follow that $|C|=|\rho(C)|\le d$, as required. 

To establish our claim it is enough to show that $|\ch|_{\rho(C)}|=2^{|C|}$. Indeed, given a subset $B\subseteq C$, by the definition of $N$-shattering, there exists $h_B\in\ch_{\Psi}$ for which
\[
\forall x\in B,h_B(x)=f_0(x)\text{ and }\forall x\in C\setminus B,h_B(x)=f_1(x)~.
\]
It follows that there exists a vector $w_B\in \mathbb R^d$ such that, for every $x\in B$,
\[
\langle w,\Psi(x,f_0(x))\rangle >\langle w,\Psi(x,f_1(x))\rangle\Rightarrow \langle w,\rho(x)\rangle >0~.
\]
Similarly, for every $x\in C\setminus B$,
\[
\langle w,\rho(x)\rangle <0~.
\]
It follows that the hypothesis $g_B\in \ch$ defined by $w\in \mathbb R^d$ label the points in $\rho(B)$ by $1$ and the points in $\rho(C\setminus B)$ by $0$. It follows that if $B_1,B_2\subseteq C$ are two different sets then $(h_{B_1})|_{\rho(C)}\ne (h_{B_2})|_{\rho(C)}$. Therefore $|\ch|_C|=2^{|C|}$ as required.
\end{proof}

\begin{remark}[Tightness of \thmref{thm:main_nat_dim}] \thmref{thm:main_nat_dim} is tight for some functions $\Psi:\cx\times\cy\to \mathbb R^d$. For example, consider the case that $\cx=[d]$, $\cy=\{\pm 1\}$ and $\Psi(x,y)=y\cdot e_x$. It is not hard to see that $\ch_{\Psi}=\cy^{\cx}$. Therefore, $\Ndim(\ch_{\Psi})=\VCdim(\ch_{\Psi})=d$. On the other hand, the theorem is not tight for every $\Psi$. For example, if $|\cx|<d$, then for every $\Psi$, $\Ndim(\ch_{\Psi})\le |\cx|<d$.
\end{remark}

\end{document}